\documentclass{article}
\usepackage{authblk}
\usepackage[utf8]{inputenc}
\usepackage[T1]{fontenc}
\usepackage[a4paper]{geometry}
\usepackage{graphicx}
\usepackage[textwidth=8em,textsize=small]{todonotes}
\usepackage{amsmath,amsthm}
\usepackage{natbib}
\usepackage{hyperref}
\usepackage[nameinlink]{cleveref}
\usepackage{bbm}
\usepackage{amsfonts}
\usepackage{mathtools}
\usepackage[titletoc,toc]{appendix}
\usepackage{url}
\usepackage[ruled, vlined, linesnumbered]{algorithm2e}
\usepackage{algpseudocode}
\usepackage{multirow}
\usepackage{booktabs}
\usepackage[normalem]{ulem}
\usepackage[absolute]{textpos}

\definecolor{darkblue}{HTML}{0c7dbb}
\definecolor{darkgreen}{rgb}{0,0.48,0.65}
\hypersetup{
    colorlinks = true,
    citecolor=darkgreen,
    filecolor=black,
    linkcolor=darkblue,
    urlcolor=blue
}
\providecommand{\keywords}[1]
{
    \small
    \textbf{\textit{Keywords---}} #1
}

\DeclareMathOperator*{\argmin}{arg\,min}
\DeclareMathOperator*{\argmax}{arg\,max}
\newtheorem{theorem}{Theorem}[section]
\newtheorem{corollary}{Corollary}[section]

\newtheorem{lemma}[theorem]{Lemma}   
\Crefname{algocf}{Algorithm}{Algorithms}

\title{Automatic Change-Point Detection in Time Series via Deep Learning}
\author[1]{Jie Li\footnote{Addresses for correspondence: Jie Li, Department of Statistics, London School of Economics and Political Science, London, WC2A 2AE.\ \textbf{Email}: j.li196@lse.ac.uk}}
\author[2]{Paul Fearnhead}
\author[1]{Piotr Fryzlewicz}
\author[1]{Tengyao Wang}
\affil[1]{Department of Statistics, London School of Economics and Political Science, London, UK}
\affil[2]{Department of Mathematics and Statistics, Lancaster University, Lancaster, UK}
\begin{document}
\maketitle
\begin{textblock}{12.5}[0,0](2,1)
    \noindent \textsf{[\textit{To be read before }The Royal Statistical Society \textit{at the Society's 2023 annual conference held in} Harrogate \textit{on Wednesday, September 6th, 2023,} the President, Dr Andrew Garrett, \textit{in the Chair}.]}
\end{textblock}
\begin{textblock}{12.5}[0,0](2,2)
    \noindent \textsf{[\textcolor{darkblue}{Accepted (with discussion), to appear}]}
\end{textblock}
\begin{abstract}
    Detecting change-points in data is challenging because of the range of possible types of change and types of behaviour of data when there is no change. Statistically efficient methods for detecting a change will depend on both of these features, and it can be difficult for a practitioner to develop an appropriate detection method for their application of interest. We show how to automatically generate new offline detection methods based on training a neural network. Our approach is motivated by many existing tests for the presence of a change-point being representable by a simple neural network, and thus a neural network trained with sufficient data should have performance at least as good as these methods. We present theory that quantifies the error rate for such an approach, and how it depends on the amount of training data. Empirical results show that, even with limited training data, its performance is competitive with the standard CUSUM-based classifier for detecting a change in mean when the noise is independent and Gaussian, and can substantially outperform it in the presence of auto-correlated or heavy-tailed noise. Our method also shows strong results in detecting and localising changes in activity based on accelerometer data.
    
    \keywords{ Automatic statistician; Classification; Likelihood-free inference; Neural networks;   Structural breaks;  Supervised learning }
    
\end{abstract}
\section{Introduction}\label{sec:Introduction}
	Detecting change-points in data sequences is of interest in many application areas such as bioinformatics~\citep{picardStatisticalApproachArray2005}, climatology~\citep{reevesReviewComparisonChangepoint2007},
	signal processing~\citep{haynesComputationallyEfficientChangepoint2017} and neuroscience~\citep{ohVarianceChangePoint2005}. 
	In this work, we are primarily concerned with the problem of offline change-point detection, where the entire data is available to the analyst beforehand. Over the past few decades, various methodologies have been extensively studied in this area, see~\citet{killickOptimalDetectionChangepoints2012,jandhyala2013inference,fryzlewiczWildBinarySegmentation2014,fryzlewiczNarrowestSignificancePursuit2021, wangHighDimensionalChange2018,truong2020selective} and references therein. Most research on change-point detection has concentrated on detecting and localising different types of change, e.g.\ change in mean~\citep{killickOptimalDetectionChangepoints2012,fryzlewiczWildBinarySegmentation2014}, variance~\citep{gaoVarianceChangePoint2019,liVarianceChangepointDetection2015}, median~\citep{fryzlewiczRobustNarrowestSignificance2021} or slope~\citep{baranowskiNarrowestoverthresholdDetectionMultiple2019,fearnheadDetectingChangesSlope2019}, amongst many others.
	
	Many change-point detection methods are based upon modelling data when there is no change and when there is a single change, and then constructing an appropriate test statistic to detect the presence of a change~\cite[e.g.][]{james1987tests,fearnhead2020relating}. The form of a good test statistic will vary with our modelling assumptions and the type of change we wish to detect. This can lead to difficulties in practice. As we use new models, it is unlikely that there will be a change-point detection method specifically designed for our modelling assumptions. Furthermore, developing an appropriate method under a complex model may be challenging, while in some applications an appropriate model for the data may be unclear but we may have substantial historical data that shows what patterns of data to expect when there is, or is not, a change.
	
	In these scenarios, currently a practitioner would need to choose the existing change detection method which seems the most appropriate for the type of data they have and the type of change they wish to detect. To obtain reliable performance, they would then need to adapt its implementation, for example tuning the choice of threshold for detecting a change. Often, this would involve applying the method to simulated or historical data.
	
	To address the challenge of automatically developing new change detection methods, this paper is motivated by the question: Can we construct new test statistics for detecting a change based only on having labelled examples of change-points?
	We show that this is indeed possible by training a neural network to classify whether or not a data set has a change of interest. This turns change-point detection in a supervised learning problem.
	
	A key motivation for our approach are results that show many common test statistics for detecting changes, such as the CUSUM test for detecting a change in mean, can be represented by simple neural networks. This means that with sufficient training data, the classifier learnt by such a neural network will give performance at least as good as classifiers corresponding to these standard tests. In scenarios where a standard test, such as CUSUM, is being applied but its modelling assumptions do not hold, we can expect the classifier learnt by the neural network to outperform it.
	
	There has been increasing recent interest in whether ideas from machine learning, and methods for classification, can be used for change-point detection. Within computer science and engineering, these include a number of methods designed for and that show promise on specific applications~\cite[e.g.][]{ahmadzadeh2018change,deryckChangePointDetection2021,GUPTA2022118260,huangetal2003syntheticdata}. Within statistics,~\cite{londschien2022random} and~\cite{lee2023training} consider training a classifier as a way to estimate the likelihood-ratio statistic for a change. However these methods train the classifier in an un-supervised way on the data being analysed, using the idea that a classifier would more easily distinguish between two segments of data if they are separated by a change-point.~\cite{changKernelChangepointDetection2019} use simulated data to help tune a kernel-based change detection method. Methods that use historical, labelled data have been used to train the tuning parameters of change-point algorithms~\cite[e.g.][]{hocking2015peakseg,liehrmann2021increased}. Also, neural networks have been employed to construct similarity scores of new observations to learned pre-change distributions for online change-point detection \citep{lee2023training}. However, we are unaware of any previous work using historical, labelled data to develop offline change-point methods.
	As such, and for simplicity, we focus on the most fundamental aspect, namely the problem of detecting a single change. Detecting and localising multiple changes is considered in~\Cref{sec:Detecting multiple changes and multiple change-types} when analysing activity data. We remark that by viewing the change-point detection problem as a classification instead of a testing problem, we aim to control the overall misclassification error rate instead of handling the Type I and Type II errors separately. In practice, asymmetric treatment of the two error types can be achieved by suitably re-weighting misclassification in the two directions in the training loss function.
	
	The method we develop has parallels with likelihood-free inference methods~\cite[]{gourieroux1993indirect,beaumont2019approximate} in that one application of our work is to use the ability to simulate from a model so as to circumvent the need to analytically calculate likelihoods. However, the approach we take is very different from standard likelihood-free methods which tend to use simulation to estimate the likelihood function itself. By comparison, we directly target learning a function of the data that can discriminate between instances that do or do not contain a change~\cite[though see][for likelihood-free methods based on re-casting the likelihood as a classification problem]{gutmann2018likelihood}.
	
	For an introduction to the statistical aspects of neural network-based classification, albeit not specifically in a change-point context, see~\citet{ripley1994neuralnetwork}.
	
	We now briefly introduce our notation. For any $n\in\mathbb{Z}^{+}$, we define $[n]\coloneqq\{1,\ldots,n\}$. We take all vectors to be column vectors unless otherwise stated.
	Let $\boldsymbol{1}_n$ be the all-one vector of length $n$.   Let $\mathbbm{1}\{\cdot\}$ represent the indicator function. The vertical symbol \( |\cdot| \) represents the absolute value or cardinality of \( \cdot \) depending on the context. For vector $\boldsymbol{x}=(x_1,\ldots,x_n)^\top$, we define its $p$-norm as $\|\boldsymbol{x}\|_p\coloneqq\big(\sum_{i=1}^n |x_i|^p\big)^{1/p}, p\geq 1$; when $p=\infty$, define $\|\boldsymbol{x}\|_\infty\coloneqq\max_i|x_i|$.
	All proofs, as well as additional simulations and real data analyses appear in the supplement.

\section{Neural networks}\label{sec:nn}
	The initial focus of our work is on the binary classification problem for whether a change-point exists in a given time series.
	We will work with multilayer neural networks with Rectified Linear Unit (ReLU) activation functions and binary output.
	The multilayer neural network consists of an input layer, hidden layers and an output layer, and can be represented by a directed acyclic graph, see~\Cref{fig:Arch-NN.eps}.
	\begin{figure}[t]
	    \centering
	    \makebox{\includegraphics[width=0.3\textwidth,height=0.25\textwidth]{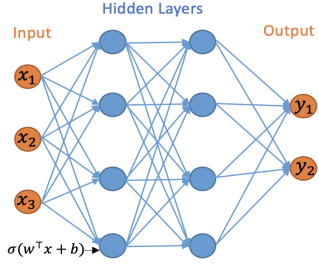}}
	    \caption{A neural network with 2 hidden layers and width vector $\mathbf{m}=(4,4)$.}\label{fig:Arch-NN.eps}
	\end{figure}
	Let  \( L\in\mathbb{Z}^{+} \) represent the number of hidden layers and \( \boldsymbol{m}={( m_{1},\ldots,m_{L} )}^{\top} \) the vector of the hidden layers widths, i.e.\ $m_i$ is the number of nodes in the $i$th hidden layer. For a neural network with $L$ hidden layers we use the convention that \( m_{0}= n\) and \( m_{L+1}=1 \). For any bias vector \( \boldsymbol{b}={ ( b_{1},b_{2},\ldots,b_{r} ) }^{\top}\in \mathbb{R}^{r} \), define the shifted activation function \( \sigma_{\boldsymbol{b}}:\mathbb{R}^{r}\to\mathbb{R}^{r}\):
	\begin{equation*}
	    \sigma_{\boldsymbol{b}}((y_1,\ldots,y_r)^{\top})=(\sigma(y_1-b_{1}),\ldots,\sigma(y_r-b_{r}))^{\top},
	\end{equation*}
	where $\sigma(x)=\max(x,0)$ is the ReLU activation function. The neural network can be mathematically represented by the composite function \( h:\mathbb{R}^{n}\to \{0,1\} \) as
	\begin{equation}\label{eq:h_fun}
	    h(\boldsymbol{x}) \coloneqq \sigma^{*}_{\lambda}W_{L}\sigma_{\boldsymbol{b}_{L}}W_{L-1}\sigma_{\boldsymbol{b}_{L-1}}\cdots W_{1}\sigma_{\boldsymbol{b}_{1}}W_{0}\boldsymbol{x},
	\end{equation}
	where \( \sigma^{*}_{\lambda}(x)=\mathbbm{1}\{x > \lambda\} \), \( \lambda>0 \) and $W_\ell\in\mathbb{R}^{m_{\ell+1}\times m_{\ell}}$ for $\ell \in\{0,\ldots,L\}$ represent the weight matrices. We define the function class \( \mathcal{H}_{L,\boldsymbol{m}} \) to be the class of functions $h(\boldsymbol{x})$ with $L$ hidden layers and width vector $\boldsymbol{m}$.
	
	The output layer in~\eqref{eq:h_fun} employs the shifted heaviside function $\sigma^{*}_{\lambda}(x)$, which is used for binary classification as the final activation function. This choice is guided by the fact that we use the 0-1 loss, which focuses on the percentage of samples assigned to the correct class, a natural performance criterion for binary classification. Besides its wide adoption in machine learning practice, another advantage of using the 0-1 loss is that it is possible to utilise the theory of the Vapnik--Chervonenkis (VC) dimension~\citep[see, e.g.][Definition~6.5]{shalev-shwartzUnderstandingMachineLearning2014} to bound the generalisation error of a binary classifier equipped with this loss; indeed, this is the approach we take in this work. The relevant results regarding the VC dimension of neural network classifiers are e.g.\ in~\cite{bartlettNearlytightVCdimensionPseudodimension2019}. As in~\cite{schmidt-hieberNonparametricRegressionUsing2020}, we work with the exact minimiser of the empirical risk. In both binary or multiclass classification, it is possible to work with other losses which make it computationally easier to minimise the corresponding risk, see e.g.~\cite{bosConvergenceRatesDeep2022}, who use a version of the cross-entropy loss. However, loss functions different from the 0-1 loss make it impossible to use VC-dimension arguments to control the generalisation error, and more involved arguments, such as those using the covering number~\citep{bosConvergenceRatesDeep2022} need to be used instead. We do not pursue these generalisations in the current work.

\section{CUSUM-based classifier and its generalisations are neural networks}\label{sec:CUSUM test and its generalisations are neural networks}
	\subsection{Change in mean}\label{subsec:Change_in_mean}
	    We initially consider the case of a single change-point with an unknown location \( \tau\in[n-1] \), $n \ge 2$, in the model
	    \begin{equation*}
	        \begin{aligned} \boldsymbol{X}&=\boldsymbol{\mu}+\boldsymbol{\xi},\\
	            \boldsymbol{\mu} & = {(\mu_{\mathrm{L}}\mathbbm{1}\{i\leq \tau\} + \mu_{\mathrm{R}}\mathbbm{1}\{i> \tau\})}_{i\in [n]} \in \mathbb{R}^n,
	        \end{aligned}
	    \end{equation*}
	    where \( \mu_{\mathrm{L}},\mu_{\mathrm{R}} \) are the unknown signal values before and after the change-point; $\boldsymbol{\xi}\sim N_{n}(0,I_{n})$. The CUSUM test is widely used to detect mean changes in univariate data. For the observation $\boldsymbol{x}$, the CUSUM transformation $\mathcal{C}:\mathbb{R}^n\to\mathbb{R}^{n-1}$ is defined as $\mathcal{C}(\boldsymbol{x}) := (\boldsymbol{v}_1^\top \boldsymbol{x}, \ldots, \boldsymbol{v}_{n-1}^\top \boldsymbol{x})^\top$, where $\boldsymbol{v}_i \coloneqq\bigl(\sqrt{\frac{n-i}{in}}\boldsymbol{1}_{i}^\top, -\sqrt{\frac{i}{(n-i)n}}\boldsymbol{1}_{n-i}^\top\bigr)^{\top}$ for $i\in[n-1]$. Here, for each $i\in[n-1]$, $(\boldsymbol{v}_i^\top \boldsymbol{x})^2$ is the log likelihood-ratio statistic for testing a change at time $i$ against the null of no change~\citep[e.g.][]{baranowskiNarrowestoverthresholdDetectionMultiple2019}.
	    For a given threshold $\lambda>0$, the classical CUSUM test for a change in the mean of the data is defined as
	    \[
	        h^{\mathrm{CUSUM}}_\lambda(\boldsymbol{x}) = \mathbbm{1}\{\|\mathcal{C}(\boldsymbol{x})\|_\infty > \lambda\}.
	    \]
	    The following lemma shows that $h^{\mathrm{CUSUM}}_\lambda(\boldsymbol{x})$ can be represented as a neural network.
	    \begin{lemma}\label{lem:CUSUMinNNet}
	        For any $\lambda > 0$, we have $h^{\mathrm{CUSUM}}_\lambda(\boldsymbol{x}) \in \mathcal{H}_{1, 2n-2}$.
	    \end{lemma}

	    The fact that the widely-used CUSUM statistic can be viewed as a simple neural network has far-reaching consequences: this means that given enough training data, a neural network architecture that permits the CUSUM-based classifier as its special case cannot do worse than CUSUM in classifying change-point versus no-change-point signals. This serves as the main motivation for our work, and a prelude to our next results.
	
	\subsection{Beyond the mean change model}
	    
	    We can generalise the simple change in mean model to allow for different types of change or for non-independent noise. In this section, we consider change-point models that can be expressed as a change in regression problem, where the model for data given a change at $\tau$ is of the form
	    \begin{equation} \label{eq:generalCP}
	        \boldsymbol{X} = \boldsymbol{Z}\boldsymbol{\beta}+\boldsymbol{c}_{\tau} \phi + \boldsymbol{\Gamma}\boldsymbol{\xi},
	    \end{equation}
	    where for some $p\geq 1$, $\boldsymbol{Z}$ is an $n\times p$ matrix of covariates for the model with no change, $\boldsymbol{c}_{\tau}$ is an $n\times 1$ vector of covariates specific to the change at $\tau$, and the parameters $\boldsymbol{\beta}$ and $\phi$ are, respectively, a $p\times 1$ vector and a scalar. The noise is defined in terms of an $n\times n$ matrix $\boldsymbol{\Gamma}$ and an $n\times 1$ vector of independent standard normal random variables, $\boldsymbol{\xi}$.
	    
	    For example, the change in mean problem has $p=1$, with $\boldsymbol{Z}$ a column vector of ones, and $\boldsymbol{c}_\tau$ being a vector whose first $\tau$ entries are zeros, and the remaining entries are ones. In this formulation $\beta$ is the pre-change mean, and $\phi$ is the size of the change. The change in slope problem~\cite[]{fearnheadDetectingChangesSlope2019} has $p=2$ with the columns of $\boldsymbol{Z}$ being a vector of ones, and a vector whose $i$th entry is $i$; and $\boldsymbol{c}_{\tau}$ has $i$th entry that is $\max\{0,i-\tau\}$. In this formulation $\boldsymbol{\beta}$ defines the pre-change linear mean, and $\phi$ the size of the change in slope.  Choosing $\boldsymbol{\Gamma}$ to be proportional to the identity matrix gives a model with independent, identically distributed noise; but other choices would allow for auto-correlation.
	    
	    The following result is a generalisation of~\Cref{lem:CUSUMinNNet}, which shows that the likelihood-ratio test for~\eqref{eq:generalCP}, viewed as a classifier, can be represented by our neural network.
	    \begin{lemma}\label{lem:generaltest}
	        Consider the change-point model~\eqref{eq:generalCP} with a possible change at $\tau \in [n-1]$. Assume further that $\boldsymbol{\Gamma}$ is invertible.  Then there is an $h^* \in\mathcal{H}_{1,2n-2}$  equivalent to the likelihood-ratio test for testing $\phi = 0$ against $\phi\neq 0$.
	    \end{lemma}
	    Importantly, this result shows that for this much wider class of change-point models, we can replicate the likelihood-ratio-based classifier for change using a simple neural network.
	    
	    Other types of changes can be handled by suitably pre-transforming the data. For instance, squaring the input data would be helpful in detecting changes in the variance and if the data followed an AR(1) structure, then changes in autocorrelation could be handled by including transformations of the original input of the form $(x_{t}x_{t+1})_{t=1,\ldots,n-1}$. On the other hand, even if such transformations are not supplied as the input, a neural network of suitable depth is able to approximate these transformations and consequently successfully detect the change~\citep[Lemma~A.2]{schmidt-hieberNonparametricRegressionUsing2020}. This is illustrated in~\Cref{fig:AR} of appendix, where we compare the performance of neural network based classifiers of various depths constructed with and without using the transformed data as inputs.

\section{Generalisation error of neural network change-point classifiers}\label{sec:Generalisation_Error}
	In~\Cref{sec:CUSUM test and its generalisations are neural networks}, we showed that CUSUM and generalised CUSUM could be represented by a neural network. Therefore, with a large enough amount of training data, a trained neural network classifier that included CUSUM, or generalised CUSUM, as a special case, would perform no worse than it on unseen data. In this section, we provide generalisation bounds for a neural network classifier for the change-in-mean problem, given a finite amount of training data. En route to this main result, stated in~\Cref{thm:VC_Generalisation_bound_new}, we provide generalisation bounds for the CUSUM-based classifier, in which the threshold has been chosen on a finite training data set.
	
	We write $P(n,\tau,\mu_{\mathrm{L}},\mu_{\mathrm{R}})$ for the distribution of the multivariate normal random vector $\boldsymbol{X} \sim N_n(\boldsymbol{\mu}, I_n)$ where $\boldsymbol{\mu} \coloneqq {(\mu_{\mathrm{L}}\mathbbm{1}\{i\leq \tau\} + \mu_{\mathrm{R}}\mathbbm{1}\{i> \tau\})}_{i\in [n]}$. Define \( \eta\coloneqq\tau/n \).~\Cref{lem:hypothesis_test} and~\Cref{cor:Generalisation} control the misclassification error of the CUSUM-based classifier.
	\begin{lemma}\label{lem:hypothesis_test}
	    Fix $\varepsilon \in (0,1)$. Suppose $\boldsymbol{X}\sim P(n,\tau,\mu_{\mathrm{L}},\mu_{\mathrm{R}})$ for some $\tau\in\mathbb{Z}^{+}$ and $\mu_{\mathrm{L}},\mu_{\mathrm{R}}\in\mathbb{R}$.
	    \begin{enumerate}
	        \item[(a)] If $\mu_{\mathrm{L}} = \mu_{\mathrm{R}}$, then
	        $
	            \mathbb{P}\bigl\{\|\mathcal{C}(\boldsymbol{X})\|_\infty > \sqrt{2\log(n/\varepsilon)}\bigr\}\leq\varepsilon.
	        $
	        \item[(b)] If $|\mu_{\mathrm{L}} - \mu_{\mathrm{R}}|\sqrt{\eta( 1-\eta)} > \sqrt{8\log(n/\varepsilon)/n}$, then
	        $
	            \mathbb{P}\bigl\{\|\mathcal{C}(\boldsymbol{X})\|_\infty \leq  \sqrt{2\log(n/\varepsilon)}\bigr\}\leq\varepsilon.
	        $
	    \end{enumerate}
	\end{lemma}
	For any $B>0$, define
	\[
	    \Theta(B)\coloneqq\left\{(\tau, \mu_{\mathrm{L}}, \mu_{\mathrm{R}}) \in [n-1] \times \mathbb{R} \times \mathbb{R}: |\mu_{\mathrm{L}} - \mu_{\mathrm{R}}|\sqrt{\tau(n-\tau)}/n \in \{0\} \cup \left(B, \infty\right)\right\}.
	\]
	Here, $|\mu_{\mathrm{L}}-\mu_{\mathrm{R}}|\sqrt{\tau(n-\tau)}/n=|\mu_{\mathrm{L}}-\mu_{\mathrm{R}}|\sqrt{\eta(1-\eta)}$ can be interpreted as the signal-to-noise ratio of the mean change problem. Thus, $\Theta(B)$ is the parameter space of data distributions where there is either no change, or a single change-point in mean whose signal-to-noise ratio is at least $B$.
	The following corollary controls the misclassification risk of a CUSUM statistics-based classifier:
	\begin{corollary}\label{cor:Generalisation}
	    Fix $B>0$. Let $\pi_0$ be any prior distribution on $\Theta(B)$, then draw $(\tau, \mu_{\mathrm{L}}, \mu_{\mathrm{R}}) \sim \pi_0$ and $\boldsymbol{X} \sim P(n, \tau, \mu_{\mathrm{L}}, \mu_{\mathrm{R}})$, and define $Y = \mathbbm{1}\{\mu_{\mathrm{L}} \neq \mu_{\mathrm{R}}\}$.  For $\lambda = B\sqrt{n}/2$, the classifier $h^{\mathrm{CUSUM}}_{\lambda}$ satisfies
	    \[
	        \mathbb{P}(h^{\mathrm{CUSUM}}_{\lambda}(\boldsymbol{X}) \neq Y) \leq ne^{-nB^2/8}.
	    \]
	\end{corollary}
	\Cref{thm:VC_Generalisation_bound_old} below, which is based on~\Cref{cor:Generalisation},~\citet[Theorem~7]{bartlettNearlytightVCdimensionPseudodimension2019} and~\citet[Corollary~3.4]{mohriFoundationsMachineLearning2012}, shows that the empirical risk minimiser in the neural network class $\mathcal{H}_{1,2n-2}$ has good generalisation properties over the class of change-point problems parameterised by $\Theta(B)$. Given training data $(\boldsymbol{X}^{(1)},Y^{(1)}),\ldots,(\boldsymbol{X}^{(N)},Y^{(N)})$ and any \( h:\mathbb{R}^n\to\{0,1\} \), we define the empirical risk of $h$ as
	\begin{equation*}
	    L_{N}(h)\coloneqq\frac{1}{N}\sum_{i=1}^N \mathbbm{1}\{Y^{(i)}\neq h(\boldsymbol{X}^{(i)})\}.
	\end{equation*}
	\begin{theorem}\label{thm:VC_Generalisation_bound_old}
	    Fix $B>0$ and let $\pi_0$ be any prior distribution on $\Theta(B)$. We draw $(\tau, \mu_{\mathrm{L}}, \mu_{\mathrm{R}}) \sim \pi_0$, $\boldsymbol{X} \sim P(n, \tau, \mu_{\mathrm{L}}, \mu_{\mathrm{R}})$, and set $Y = \mathbbm{1}\{\mu_{\mathrm{L}} \neq \mu_{\mathrm{R}}\}$. Suppose that the training data $\mathcal{D}:= \bigl((\boldsymbol{X}^{(1)},Y^{(1)}),\ldots,(\boldsymbol{X}^{(N)},Y^{(N)})\bigr)$ consist of independent copies of $(\boldsymbol{X},Y)$ and \sloppy  $h_{\mathrm{ERM}} \coloneqq\argmin_{h\in \mathcal{H}_{1,2n-2}}L_{N}(h)$ is the empirical risk minimiser.
	    There exists a universal constant $C>0$ such that for any $\delta \in(0,1)$,~\eqref{eq:bound_old} holds with probability \( 1- \delta \).
	    \begin{equation}\label{eq:bound_old}
	        \mathbb{P}(h_{\mathrm{ERM}}(\boldsymbol{X})\neq Y\mid \mathcal{D}) \leq ne^{-nB^2/8}+ C\sqrt{\frac{n^2\log(n)\log(N)+\log(1/\delta)}{N}}.
	    \end{equation}
	\end{theorem}
	The theoretical results derived for the neural network-based classifier, here and below, all rely on the fact that the training and test data are drawn from the same distribution. 
	However, we observe that in practice, even when the training and test sets have different error distributions, neural network-based classifiers still provide accurate results on the test set; see our discussion of~\Cref{fig:simulation_result} in~\Cref{sec:Simulation_Study} for more details.
	The misclassification error in~\eqref{eq:bound_old} is bounded by two terms.
	The first term represents the misclassification error of CUSUM-based classifier, see~\Cref{cor:Generalisation}, and the second term depends on the complexity of the neural network class measured in its VC dimension.~\Cref{thm:VC_Generalisation_bound_old} suggests that for training sample size $N\gg n^2\log n$, a well-trained single-hidden-layer neural network with $2n-2$ hidden nodes would have comparable performance to that of the CUSUM-based classifier.
	However, as we will see in~\Cref{sec:Simulation_Study}, in practice, a much smaller training sample size $N$ is needed for the neural network to be competitive in the change-point detection task. This is because the $2n-2$ hidden layer nodes in the neural network representation of $h^{\mathrm{CUSUM}}_\lambda$ encode the components of the CUSUM transformation $(\pm \boldsymbol{v}_t^\top \boldsymbol{x}: t\in[n-1])$, which are highly correlated.

	By suitably pruning the hidden layer nodes, we can show that a single-hidden-layer neural network with $O(\log n)$ hidden nodes is able to represent a modified version of the CUSUM-based classifier with  essentially the same misclassification error. More precisely, let $Q:=\lfloor \log_2(n/2)\rfloor$ and write $T_0 := \{2^q: 0\leq q\leq Q\} \cup \{n - 2^q: 0\leq q\leq Q\}$.
	We can then define
	\[
	    h^{\mathrm{CUSUM}_*}_{\lambda^{*}}(\boldsymbol{X}) = \mathbbm{1}\Bigl\{\max_{t\in T_0} |\boldsymbol{v}_t^\top \boldsymbol{X}| > \lambda^{*}\Bigr\}.
	\]
	By the same argument as in~\Cref{lem:CUSUMinNNet}, we can show that \( h^{\mathrm{CUSUM}_*}_{\lambda^{*}} \in \mathcal{H}_{1,4\lfloor\log_2(n)\rfloor} \) for any $\lambda^* > 0$. The following Theorem shows that high classification accuracy can be achieved under a weaker training sample size condition compared to~\Cref{thm:VC_Generalisation_bound_old}.

	\begin{theorem}\label{thm:VC_Generalisation_bound_new}
	    Fix $B>0$ and let the training data $\mathcal{D}$
	    be generated as in~\Cref{thm:VC_Generalisation_bound_old}. Let  $h_{\mathrm{ERM}} \coloneqq\argmin_{h\in \mathcal{H}_{L,\boldsymbol{m}}}L_{N}(h)$ be the empirical risk minimiser for a neural network with $L\geq 1$ layers and $\boldsymbol{m} = (m_1,\ldots,m_L)^\top$ hidden layer widths. If $m_1\geq 4\lfloor\log_2(n)\rfloor$ and $m_r m_{r+1} = O(n\log n)$ for all $r\in[L-1]$, then there exists a universal constant $C>0$ such that for any $\delta \in(0,1)$,~\eqref{eq:bound_new} holds with probability \( 1- \delta \).
	    \begin{equation}\label{eq:bound_new}
	        \mathbb{P}(h_{\mathrm{ERM}}(\boldsymbol{X})\neq Y\mid\mathcal{D}) \leq 2\lfloor\log_2(n)\rfloor e^{-nB^2/24}+ C\sqrt{\frac{L^2 n\log^2(Ln)\log(N)+\log(1/\delta)}{N}}.
	    \end{equation}
	\end{theorem}
	\Cref{thm:VC_Generalisation_bound_new} generalises the single hidden layer neural network representation in~\Cref{thm:VC_Generalisation_bound_old} to multiple hidden layers.
	In practice, multiple hidden layers help keep the misclassification error rate low even when $N$ is small, see the numerical study in~\Cref{sec:Simulation_Study}.
	\Cref{thm:VC_Generalisation_bound_old,thm:VC_Generalisation_bound_new} are examples of how to derive generalisation errors of a neural network-based classifier in the change-point detection task. The same workflow can be employed in other types of changes, provided that suitable representation results of likelihood-based tests in terms of neural networks (e.g.~\Cref{lem:generaltest}) can be obtained. In a general result of this type, the generalisation error of the neural network will again be bounded by a sum of the error of the likelihood-based classifier together with a term originating from the VC-dimension bound of the complexity of the neural network architecture.
	
	We further remark that for simplicity of discussion, we have focused our attention on data models where the noise vector $\boldsymbol{\xi} = \boldsymbol{X}-\mathbb{E}\boldsymbol{X}$ has independent and identically distributed normal components. However, since CUSUM-based tests are available for temporally correlated or sub-Weibull data, with suitably adjusted test threshold values, the above theoretical results readily generalise to such settings. See~\Cref{Thm:Temporal,Thm:SubWeibull} in appendix for more details.

\section{Numerical study}\label{sec:Simulation_Study}
	We now investigate empirically our approach of learning a change-point detection method by training a neural network. Motivated by the results from the previous section we will fit a neural network with a single layer and consider how varying the number of hidden layers and the amount of training data affects performance. We will compare to a test based on the CUSUM statistic, both for scenarios where the noise is independent and Gaussian, and for scenarios where there is auto-correlation or heavy-tailed noise. The CUSUM test can be sensitive to the choice of threshold, particularly when we do not have independent Gaussian noise, so we tune its threshold based on training data.
	
	When training the neural network, we first standardise the data onto \( [0,1] \), i.e.\ \( \tilde{\boldsymbol{x}}_{i}=((x_{ij}-x_{i}^{\mathrm{min}})/(x_{i}^{\mathrm{max}}-x_{i}^{\mathrm{min}}))_{j\in[n]} \) where \( x_{i}^{\mathrm{max}}:=\max_{j}x_{ij}, x_{i}^{\mathrm{min}}:=\min_j x_{ij} \).
	This  makes the neural network procedure invariant to either adding a constant to the data or scaling the data by a constant, which are natural properties to require. We train the neural network by minimising the cross-entropy loss on the training data. We run training for 200 epochs with a batch size of 32 and a learning rate of 0.001 using the Adam optimiser~\citep{kingma2014adam}. These hyperparameters are chosen based on a training dataset with cross-validation, more details can be found in~\Cref{sec:simulation_and_result}.

	We generate our data as follows. Given a sequence of length $n$, we draw $\tau\sim \mathrm{Unif}\{2,\ldots,n-2\}$, set $\mu_{\mathrm{L}}=0$ and draw $\mu_{\mathrm{R}}| \tau \sim \mathrm{Unif}([-1.5b, -0.5b]\cup[0.5b, 1.5b])$, where $b:=\sqrt{\frac{8n\log(20n)}{\tau(n-\tau)}}$ is chosen in line with~\Cref{lem:hypothesis_test} to ensure a good range of signal-to-noise ratios. We then generate $\boldsymbol{x}_1 = (\mu_{\mathrm{L}}\mathbbm{1}_{\{t\leq \tau\}} + \mu_{\mathrm{R}}\mathbbm{1}_{\{t> \tau\}} + \varepsilon_t)_{t\in[n]}$, with the noise $(\varepsilon_t)_{t\in[n]}$ following an $\mathrm{AR}(1)$ model with possibly time-varying autocorrelation $\varepsilon_t|\rho_t =\xi_1$ for $t=1$ and $\rho_t \varepsilon_{t-1} + \xi_t$ for $t\geq 2$, where $(\xi_t)_{t\in[n]}$ are independent, possibly heavy-tailed noise. The autocorrelations $\rho_t$ and innovations $\xi_t$ are from one of the three scenarios:
	{\small\begin{enumerate}
	    \item[S1:] \(n=100\), \( N\in\{ 100,200,\ldots,700\} \), \( \rho_{t}=0 \) and \( \xi_{t}\sim N(0,1) \).
	    \item[S1$^{\prime}$:] \(n=100\), \( N\in\{ 100,200,\ldots,700\} \), \( \rho_{t}= 0.7 \) and \( \xi_{t}\sim N(0,1) \).
	    \item[S2:] \(n=100\), \( N\in\{ 100,200,\ldots,1000\} \), \( \rho_{t}\sim \mathrm{Unif}([0,1]) \) and \( \xi_{t}\sim N(0,2) \).
	    \item[S3:] \(n=100\), \( N\in\{ 100,200,\ldots,1000\} \), \( \rho_{t}=0 \) and \( \xi_{t}\sim \text{Cauchy}(0,0.3) \).
	\end{enumerate}}
	The above procedure is then repeated $N/2$ times to generate independent sequences \sloppy $\boldsymbol{x}_1,\ldots, \boldsymbol{x}_{N/2}$ with a single change, and the associated labels are $(y_1,\ldots,y_{N/2})^\top = \mathbf{1}_{N/2}$. We then repeat the process another $N/2$ times with $\mu_{\mathrm{R}}=\mu_{\mathrm{L}}$ to generate  sequences without changes $\boldsymbol{x}_{N/2+1},\ldots, \boldsymbol{x}_{N}$ with $(y_{N/2+1},\ldots,y_{N})^\top = \mathbf{0}_{N/2}$. The data with and without change $(\boldsymbol{x}_i, y_i)_{i\in[N]}$ are combined and randomly shuffled to form the training data. The test data are generated in a similar way, with a sample size $N_{\mathrm{test}} = 30000$ and the slight modification that $\mu_{\mathrm{R}}|\tau \sim \mathrm{Unif}([-1.75b, -0.25b]\cup[0.25b, 1.75b])$ when a change occurs. We note that the test data is drawn from the same distribution as the training set, though potentially having changes with signal-to-noise ratios outside the range covered by the training set. We have also conducted robustness studies to investigate the effect of training the neural networks on scenario S1 and test on S1$^{\prime}$, S2 or S3. Qualitatively similar results to~\Cref{fig:simulation_result} have been obtained in this misspecified setting (see~\Cref{fig:simulation_resultA2Others} in appendix).
	\begin{figure}[ht]
	    \begin{minipage}{1\linewidth}
	        \makebox[.5\linewidth]{\includegraphics[width=.45\linewidth, height=0.27\textwidth]{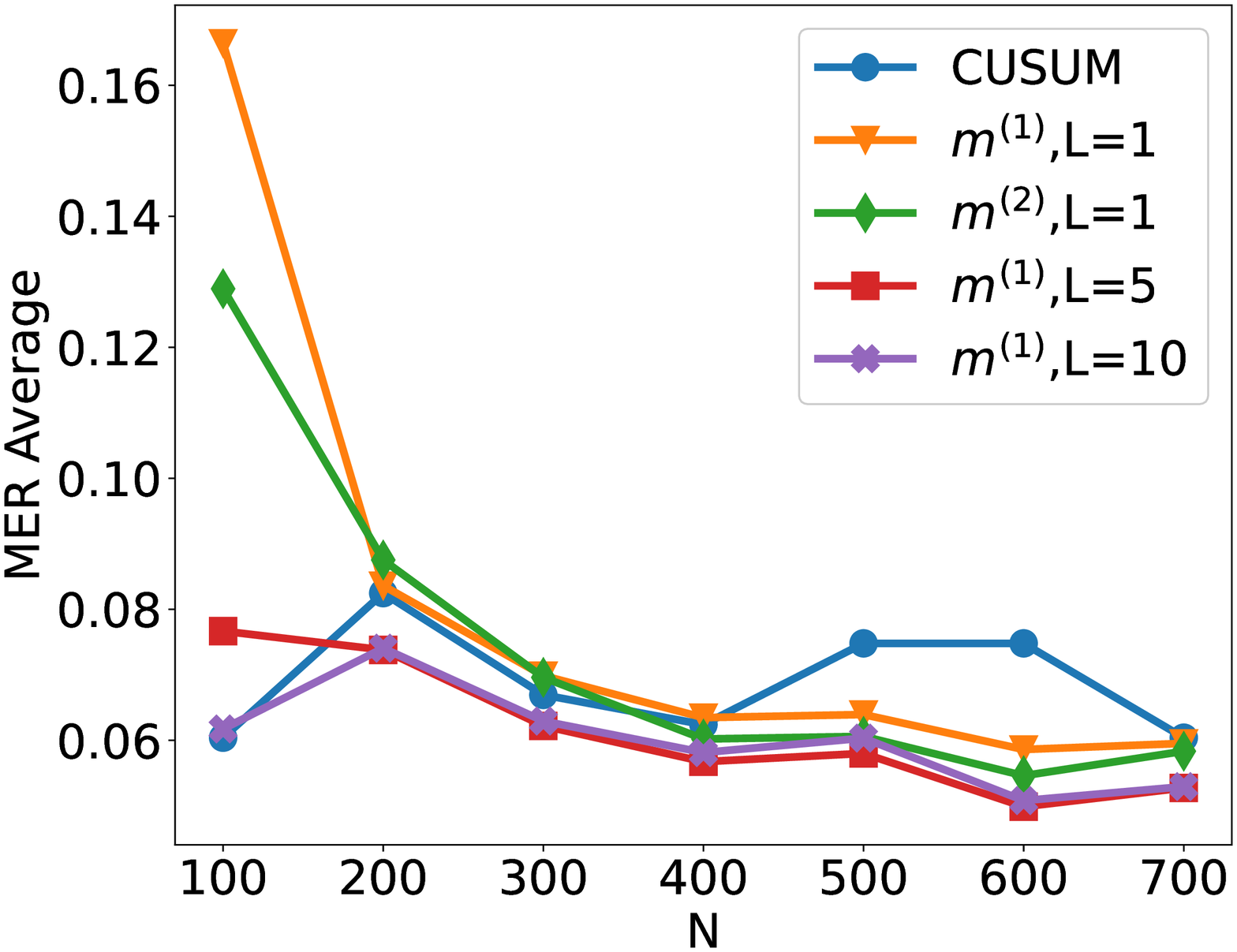}}%
	        \makebox[.5\linewidth]{\includegraphics[width=.45\linewidth, height=0.27\textwidth]{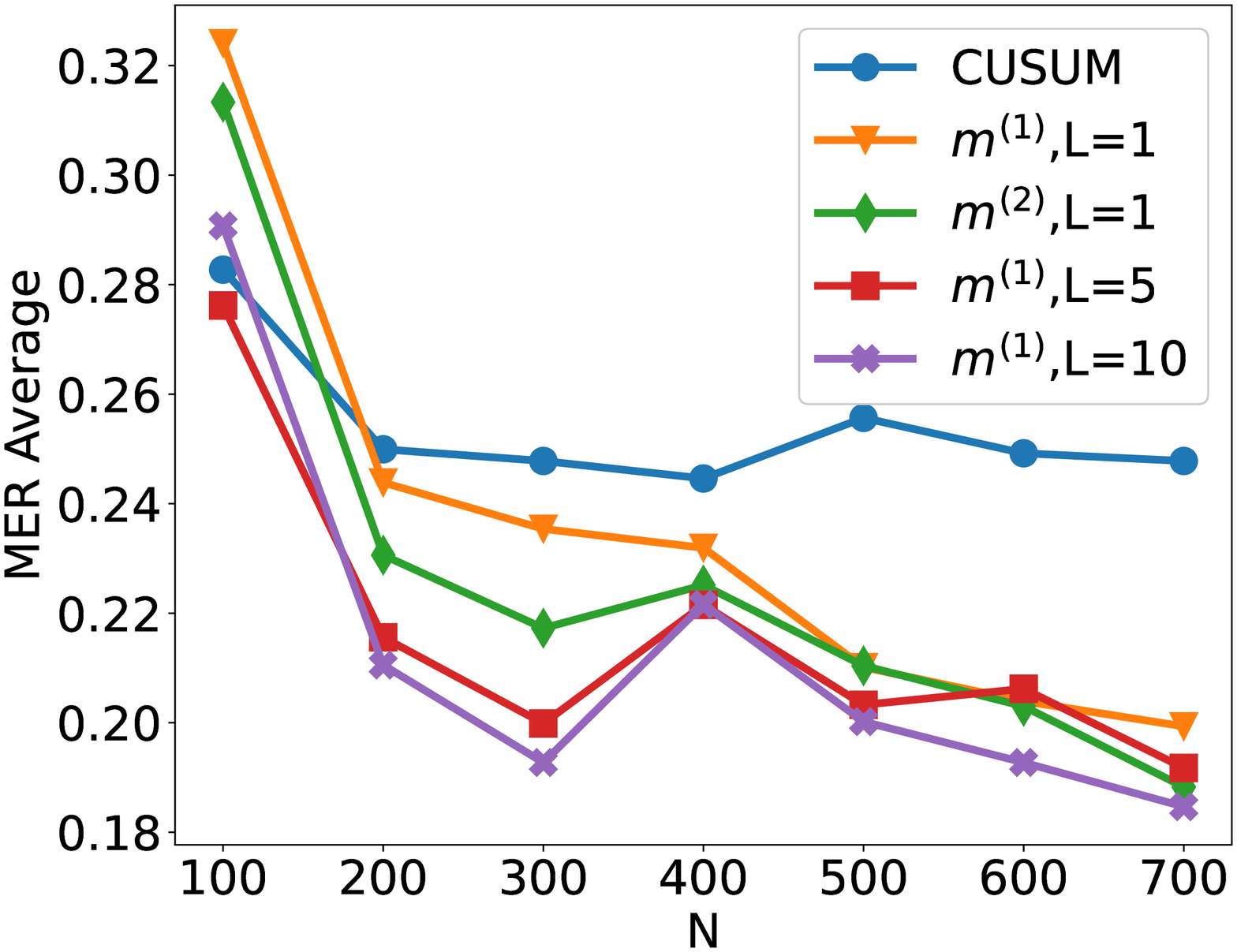}}
	        \makebox[.5\linewidth]{\small (a) Scenario S1 with \( \rho_{t}= 0 \)}%
	        \makebox[.5\linewidth]{\small (b) Scenario S1$^{\prime}$ with \( \rho_{t}= 0.7 \)}%
	        
	        \medskip
	        
	        \makebox[.5\linewidth]{\includegraphics[width=.45\linewidth, height=0.27\textwidth]{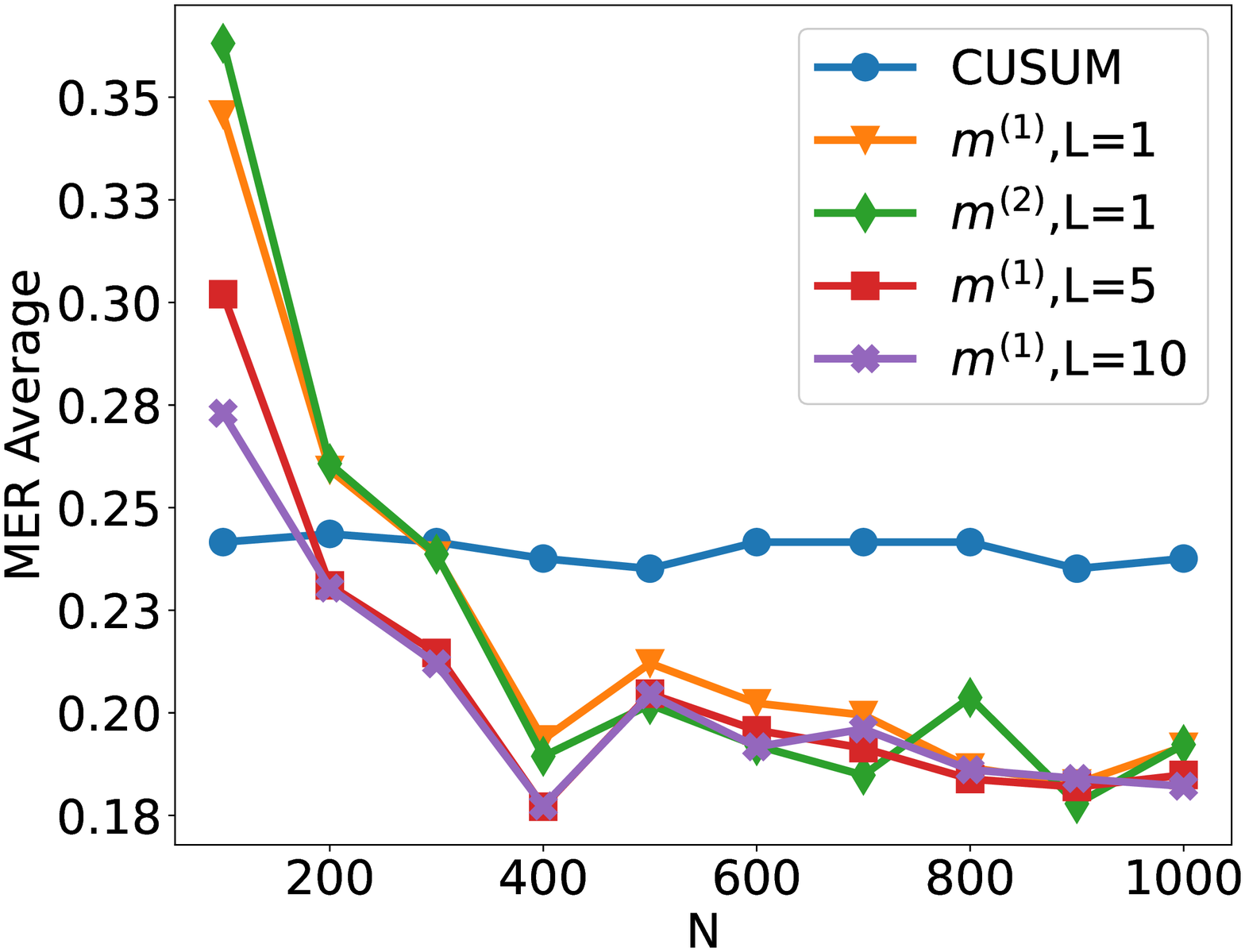}}
	        \makebox[.5\linewidth]{\includegraphics[width=.45\linewidth, height=0.27\textwidth]{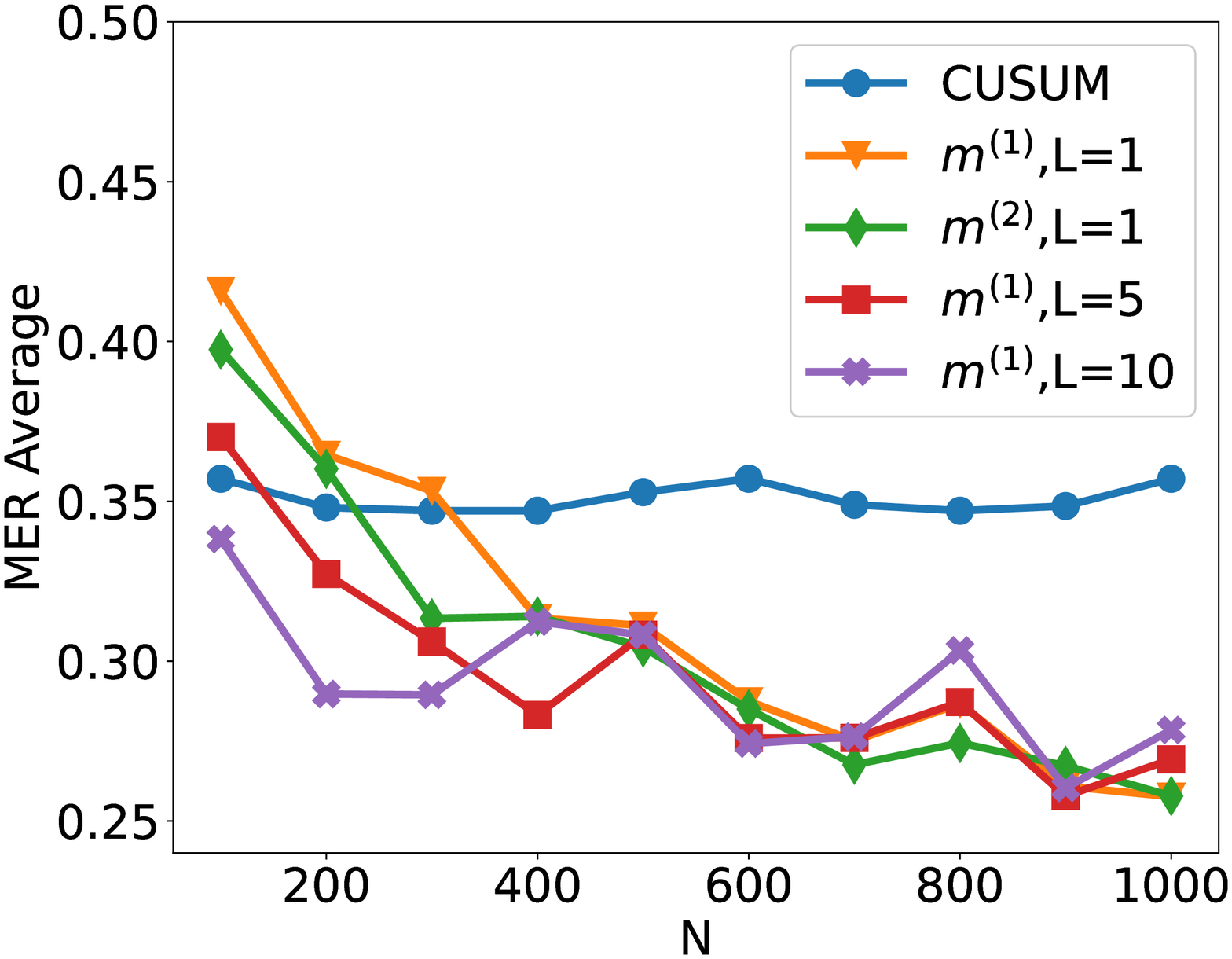}}
	        \makebox[.5\linewidth]{\small (c) Scenario S2 with \( \rho_{t}\sim \text{Unif}([0,1]) \)}%
	        \makebox[.5\linewidth]{\small (d) Scenario S3 with Cauchy noise}%
	    \end{minipage}
	    \caption{Plot of the test set MER, computed on a test set of size $N_{\mathrm{test}}=30000$,  against training sample size $N$ for detecting the existence of a change-point on data series of length $n=100$. We compare the performance of the CUSUM test and neural networks from four function classes: $\mathcal{H}_{1,m^{(1)}}$,$\mathcal{H}_{1,m^{(2)}}$, $\mathcal{H}_{5,m^{(1)}\mathbf{1}_{5}}$ and $\mathcal{H}_{10,m^{(1)}\mathbf{1}_{10}}$ where $m^{(1)} = 4\lfloor\log_2(n)\rfloor$ and $m^{(2)} = 2n-2$ respectively under scenarios S1, S1$^{\prime}$, S2 and S3 described in~\Cref{sec:Simulation_Study}. }\label{fig:simulation_result}
	\end{figure}
	We compare the performance of the CUSUM-based classifier with the threshold cross-validated on the training data with neural networks from four function classes: $\mathcal{H}_{1,m^{(1)}}$,$\mathcal{H}_{1,m^{(2)}}$, $\mathcal{H}_{5,m^{(1)}\mathbf{1}_{5}}$ and $\mathcal{H}_{10,m^{(1)}\mathbf{1}_{10}}$ where $m^{(1)} = 4\lfloor\log_2(n)\rfloor$ and $m^{(2)} = 2n-2$ respectively (cf.~\Cref{thm:VC_Generalisation_bound_new} and~\Cref{lem:CUSUMinNNet}).~\Cref{fig:simulation_result} shows the test misclassification error rate (MER) of the four procedures in the four scenarios S1, S1$^{\prime}$, S2 and S3. We observe that when data are generated with independent Gaussian noise (~\Cref{fig:simulation_result}(a)), the trained neural networks with $m^{(1)}$ and $m^{(2)}$ single hidden layer nodes attain very similar test MER compared to the CUSUM-based classifier. This is in line with our~\Cref{thm:VC_Generalisation_bound_new}. More interestingly, when noise has either autocorrelation (~\Cref{fig:simulation_result}(b, c)) or heavy-tailed distribution (~\Cref{fig:simulation_result}(d)), trained neural networks with $(L,\mathbf{m})$: $(1,m^{(1)})$, $(1,m^{(2)})$, $(5,m^{(1)}\mathbf{1}_{5})$ and $(10,m^{(1)}\mathbf{1}_{10})$ outperform the CUSUM-based classifier, even after we have optimised the threshold choice of the latter. In addition, as shown in~\Cref{fig:S3WilcoxonTruncation.eps} in the online supplement, when the first two layers of the network are set to carry out truncation, which can be seen as a composition of two ReLU operations, the resulting neural network outperforms the Wilcoxon statistics-based classifier~\citep{dehling2013changepoint}, which is a standard benchmark for change-point detection in the presence of heavy-tailed noise. Furthermore, from~\Cref{fig:simulation_result}, we see that increasing $L$ can significantly reduce the average MER when $N\leq 200$. Theoretically, as the number of layers $L$ increases, the neural network is better able to approximate the optimal decision boundary, but it becomes increasingly difficult to train the weights due to issues such as vanishing gradients~\citep{heDeepResidualLearning2016}. A combination of these considerations leads us to develop deep neural network architecture with residual connections for detecting multiple changes and multiple change types in~\Cref{sec:Detecting multiple changes and multiple change-types}.

\section{Detecting multiple changes and multiple change types -- case study}\label{sec:Detecting multiple changes and multiple change-types}
	
	From the previous section, we see that single and multiple hidden layer neural networks can represent CUSUM or generalised CUSUM tests and may perform better than likelihood-based test statistics when the model is misspecified. This prompted us to seek a general network architecture that can detect, and even classify, multiple types of change.
	Motivated by the similarities between signal processing and image recognition, we employed a deep convolutional neural network (CNN)~\citep{yamashita2018convolutional} to learn the various features of multiple change-types. However, stacking more CNN layers cannot guarantee a better network because of vanishing gradients in training~\citep{heDeepResidualLearning2016}. Therefore, we adopted the residual block structure~\citep{heDeepResidualLearning2016} for our neural network architecture. After experimenting with various architectures with different numbers of residual blocks and fully connected layers on synthetic data, we arrived at a network architecture with 21 residual blocks followed by a number of fully connected layers.~\Cref{fig:Arch-ResNet.eps} shows an overview of the architecture of the final general-purpose deep neural network for change-point detection. The precise architecture and training methodology of this network $\widehat{NN}$ can be found in~\Cref{sec:More_Details_of_Numerical_Study_and_Real_Data_Analysis}. Neural Architecture Search (NAS) approaches~\citep[see][Section 2.4.3]{paab2023foundation} offer principled ways of selecting neural architectures. Some of these approaches could be made applicable in our setting.
	
	We demonstrate the power of our general purpose change-point detection network in a numerical study. We train the network on $N=10000$ instances of data sequences generated from a mixture of no change-point in mean or variance, change in mean only, change in variance only, no-change in a non-zero slope and change in slope only, and compare its classification performance on a test set of size $2500$ against that of oracle likelihood-based classifiers (where we pre-specify whether we are testing for change in mean, variance or slope) and adaptive likelihood-based classifiers (where we combine likelihood based tests using the Bayesian Information Criterion). Details of the data-generating mechanism and classifiers can be found in~\Cref{sec:simulation_and_result}. The classification accuracy of the three approaches in weak and strong signal-to-noise ratio settings are reported in~\Cref{tab:The accuracy of LR and NN}. We see that the neural network-based approach achieves similar classification accuracy as adaptive likelihood based method for weak SNR and higher classification accuracy than the adaptive likelihood based method for strong SNR.\ We would not expect the neural network to outperform the oracle likelihood-based classifiers as it has no knowledge of the exact change-type of each time series.\

	\begin{table}
	    \caption{\label{tab:The accuracy of LR and NN} Test classification accuracy of oracle likelihood-ratio based method (LR$^{\mathrm{oracle}}$), adaptive likelihood ratio method (LR$^{\mathrm{adapt}}$) and our residual neural network (NN) classifier for setups with weak and strong signal-to-noise ratios (SNR). Data are generated as a mixture of no change-point in mean or variance (Class 1), change in mean only (Class 2), change in variance only (Class 3), no-change in a non-zero slope (Class 4), change in slope only (Class 5). We report the true positive rate of each class and the accuracy in the last row.}
	    \centering
	    \fbox{%
	        \begin{tabular}{*{7}c}
	            \multirow{2}{*}{}&  \multicolumn{3}{c}{Weak SNR} &  \multicolumn{3}{c}{Strong SNR}\\
	            \cmidrule(lr){2-4}\cmidrule(lr){5-7}
	            & LR$^{\mathrm{oracle}}$ & LR$^{\mathrm{adapt}}$  & NN & LR$^{\mathrm{oracle}}$ & LR$^{\mathrm{adapt}}$ & NN \\ \hline
	            Class 1 &	0.9787 & 0.9457 & 0.8062 &   0.9787 & 0.9341 & 0.9651\\
	            Class 2	&   0.8443 &  0.8164 & 0.8882 &  1.0000 & 0.7784 &  0.9860\\
	            Class 3 & 	0.8350 &  0.8291 & 0.8585 &  0.9902 & 0.9902 &  0.9705\\
	            Class 4 & 	0.9960 &  0.9453 & 0.8826 &  0.9980 & 0.9372 &  0.9312\\
	            Class 5 & 	0.8729 &  0.8604 & 0.8353 &  0.9958 & 0.9917 &  0.9147\\
	            Accuracy & 	0.9056 &  0.8796 & 0.8660 &  0.9924 & 0.9260 &  0.9672
	        \end{tabular}}
	\end{table}

	We now consider an application to detecting different types of change.
	\href{http://hasc.jp/hc2011/index-en.html}{The HASC (Human Activity Sensing Consortium) project} data contain motion sensor measurements during a sequence of human activities, including ``stay'', ``walk'', ``jog'', ``skip'', ``stair up'' and ``stair down''. Complex changes in sensor signals occur during transition from one activity to the next (see~\Cref{fig:figures/HASC2011.eps}). We have 28 labels in HASC data, see~\Cref{fig:figures/label_dict.eps} in appendix. To agree with the dimension of the output, we drop two dense layers ``Dense(10)'' and ``Dense(20)'' in~\Cref{fig:Arch-ResNet.eps}.
	The resulting network can be effectively applied for change-point detection in sensory signals of human activities, and can achieve high accuracy in change-point classification tasks (\Cref{fig:figures/RealDataKS25+acc.eps} in appendix).
	
	Finally, we remark that our neural network-based change-point detector can be utilised to detect multiple change-points.~\Cref{alg:change_point_detection} outlines a general scheme for turning a change-point classifier into a location estimator, where we employ an idea similar to that of MOSUM~\citep{eichingerMOSUMProcedureEstimation2018} and repeatedly apply a classifier $\psi$ to data from a sliding window of size $n$. Here, we require $\psi$ applied to each data segment $\boldsymbol{X}^*_{[i,i+n)}$ to output both the class label $L_i = 0$ or $1$ if no change or a change is predicted and the corresponding probability $p_i$ of having a change. In our particular example, for each data segment $\boldsymbol{X}^*_{[i,i+n)}$ of length $n=700$, we define $\psi(\boldsymbol{X}^*_{[i,i+n)}) = 0$  if $\widehat{NN}(\boldsymbol{X}^*_{[i,i+n)})$ predicts a class label in $\{0,4,8,12,16,22\}$ (see~\Cref{fig:figures/label_dict.eps} in appendix) and 1 otherwise. The thresholding parameter $\gamma \in \mathbb{Z}^{+}$ is chosen to be $1/2$.
	\begin{algorithm}[htbp]
	    \SetAlgoLined
	    \caption{Algorithm for change-point localisation}\label{alg:change_point_detection}
	    \KwIn{new data $\boldsymbol{x}_1^*,\ldots,\boldsymbol{x}_{n^*}^*\in\mathbb{R}^d$,  a trained classifier $\psi: \mathbb{R}^{d\times n}\to \{0,1\}$,  $\gamma > 0$.}
	    Form $\boldsymbol{X}_{[i,i+n)}^* := (\boldsymbol{x}_i^*,\ldots,\boldsymbol{x}_{i+n-1})$ and compute $L_{i} \gets \psi(\boldsymbol{X}^{*}_{[i,i+n)})$ for all $i = 1,\ldots,n^*-n+1$\;
	    Compute $\bar L_i \gets n^{-1}\sum_{j=i-n+1}^i L_j$ for $i=n,\ldots,n^*-n+1$\;
	    Let $\{[s_1,e_1],\ldots,[s_{\hat \nu},e_{\hat \nu}]\}$ be the set of all maximal segments such that $\bar L_i \geq \gamma$ for all $i\in [s_r,e_r]$, $r\in[\hat \nu]$ \;
	    Compute $\hat \tau_r \gets \argmax_{i\in [s_r,e_r]} \bar L_i$ for all $r\in[\hat\nu]$\;
	    \KwOut{Estimated change-points $\hat\tau_1,\ldots,\hat\tau_{\hat\nu}$}
	\end{algorithm}
	\Cref{fig:figures/RealDataCPD.eps} illustrates the result of multiple change-point detection in HASC data which provides evidence that the trained neural network can detect both the multiple change-types and multiple change-points.
	\begin{figure}[ht]
	    \centering
	    \makebox{\includegraphics[width=0.78\textwidth,height=0.3\textwidth]{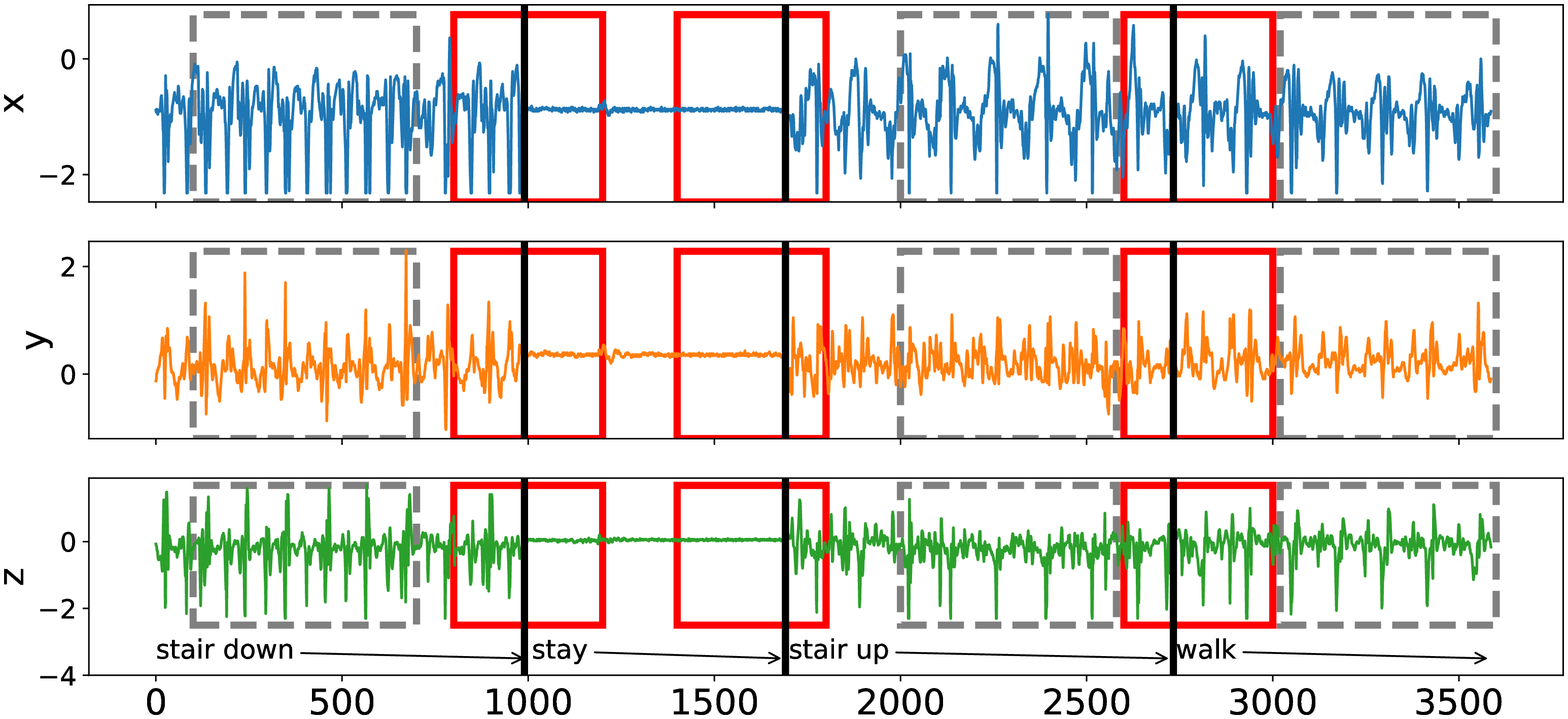}}
	    \caption{The sequence of accelerometer data in \( x, y \) and \( z \) axes. From left to right, there are 4 activities: ``stair down'', ``stay'', ``stair up'' and ``walk'', their change-points are 990, 1691, 2733 respectively marked by black solid lines. The grey rectangles represent the group of ``no-change'' with labels: ``stair down'',  ``stair up'' and ``walk''; The red rectangles represent the group of ``one-change'' with labels: ``stair down$\to$stay'', ``stay$\to$stair up'' and ``stair up$\to$walk''.}\label{fig:figures/HASC2011.eps}
	\end{figure}
	\begin{figure}[ht]
	    \centering
	    \makebox{\includegraphics[width=0.78\textwidth,height=0.3\textwidth]{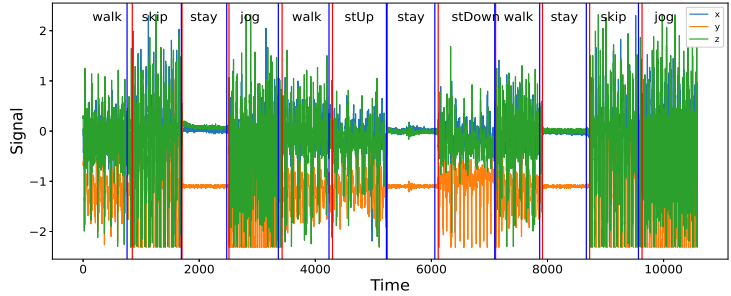}}
	    \caption{Change-point detection in HASC data. The red vertical lines represent the underlying change-points, the blue vertical lines represent the estimated change-points.  More details on multiple change-point detection  can be found in~\Cref{sec:More_Details_of_Numerical_Study_and_Real_Data_Analysis}.}\label{fig:figures/RealDataCPD.eps}
	\end{figure}

\section{Discussion}\label{sec:Discussion}
	
	Reliable testing for change-points and estimating their locations, especially in the presence of multiple change-points, other heterogeneities or untidy data, is typically a difficult problem for the applied statistician: they need to understand what type of change is sought, be able to characterise it mathematically, find a satisfactory stochastic model for the data, formulate the appropriate statistic, and fine-tune its parameters. This makes for a long workflow, with scope for errors at its every stage.
	
	In this paper, we showed how a carefully constructed statistical learning framework could automatically take over some of those tasks, and perform many of them `in one go' when provided with examples of labelled data. This turned the change-point detection problem into a supervised learning problem, and meant that the task of learning the appropriate test statistic and fine-tuning its parameters was left to the `machine' rather than the human user.
	
	The crucial question was that of choosing an appropriate statistical learning framework. The key factor behind our choice of neural networks was the discovery that the traditionally-used likelihood-ratio-based change-point detection statistics could be viewed as simple neural networks, which (together with bounds on generalisation errors beyond the training set) enabled us to formulate and prove the corresponding learning theory. However, there are a plethora of other excellent predictive frameworks, such as XGBoost, LightGBM or Random Forests~\citep{chenXGBoost2016,ke2017lightgbm,breiman2001random} and it would be of interest to establish whether and why they could or could not provide a viable alternative to neural nets here. Furthermore, if we view the neural network as emulating the likelihood-ratio test statistic, in that it will create test statistics for each possible location of a change and then amalgamate these into a single classifier, then we know that test statistics for nearby changes will often be similar. This suggests that imposing some smoothness on the weights of the neural network may be beneficial.
	
	A further challenge is to develop methods that can adapt easily to input data of different sizes, without having to train a different neural network for each input size. For changes in the structure of the mean of the data, it may be possible to use ideas from functional data analysis so that we pre-process the data, with some form of smoothing or imputation, to produce input data of the correct length.
	
	If historical labelled examples of change-points, perhaps provided by subject-matter experts (who are not necessarily statisticians) are not available, one question of interest is whether simulation can be used to obtain such labelled examples artificially, based on (say) a single dataset of interest. Such simulated examples would need to come in two flavours: one batch `likely containing no change-points' and the other containing some artificially induced ones. How to simulate reliably in this way is an important problem, which this paper does not solve. Indeed, we can envisage situations in which simulating in this way may be easier than solving the original unsupervised change-point problem involving the single dataset at hand, with the bulk of the difficulty left to the `machine' at the learning stage when provided with the simulated data.
	
	For situations where there is no historical data, but there are statistical models, one can obtain training data by simulation from the model. In this case, training a neural network to detect a change has similarities with likelihood-free inference methods in that it replaces analytic calculations associated with a model by the ability to simulate from the model. It is of interest  whether ideas from that area of statistics can be used here.
	
	The main focus of our work was on testing for a single offline change-point, and we treated location estimation and extensions to multiple-change scenarios only superficially, via the heuristics of testing-based estimation in~\Cref{sec:Detecting multiple changes and multiple change-types}.  Similar extensions can be made to the online setting once the neural network is trained, by retaining the final $n$ observations in an online stream in memory and applying our change-point classifier sequentially. One question of interest is whether and how these heuristics can be made more rigorous: equipped with an offline classifier only, how can we translate the theoretical guarantee of this offline classifier to that of the corresponding location estimator or online detection procedure?
	In addition to this approach, how else can a neural network, however complex, be trained to estimate locations or detect change-points sequentially? In our view, these questions merit further work. 
	

\section*{Availability of data and computer code} The data underlying this article are available in~\href{http://hasc.jp/hc2011/index-en.html}{http://hasc.jp/hc2011/index-en.html}. The computer code and algorithm are available in \href{https://pypi.org/project/autocpd/}{Python Package: AutoCPD}.
\section*{Acknowledgement}
	This work was supported by the High End Computing Cluster at Lancaster University, and EPSRC grants
	EP/V053590/1, EP/V053639/1 and EP/T02772X/1. We highly appreciate Yudong Chen's contribution to debug our Python scripts and improve their readability.
\section*{Conflicts of Interest}\label{sec:Conflict_of_Interest}
	We have no conflicts of interest to disclose.
	\small\bibliographystyle{chicago}
	\bibliography{reference}
	\clearpage
	\newpage
	\appendix
	This is the appendix for the main paper~\citet*{CPDNN2022}, hereafter referred to as the main text. We present  proofs of our main lemmas and theorems. Various technical details, results of numerical study and real data analysis are also listed here.
\section{Proofs}
	\subsection{The proof of~\Cref{lem:CUSUMinNNet}}
	    Define $W_0 \coloneqq (\boldsymbol{v}_1,\ldots,\boldsymbol{v}_{n-1},-\boldsymbol{v}_1,\ldots,-\boldsymbol{v}_{n-1})^\top$ and $W_1 \coloneqq \boldsymbol{1}_{2n-2}$, $\boldsymbol{b}_1\coloneqq\lambda\boldsymbol{1}_{2n-2}$ and $b_2\coloneqq0$. Then $h(\boldsymbol{x}) \coloneqq \sigma^*_{b_2} W_1 \sigma_{\boldsymbol{b}_1}W_0 \boldsymbol{x} \in\mathcal{H}_{1,2n-2}$ can be rewritten as
	    \[
	        h(\boldsymbol{x}) = \mathbbm{1}\biggl\{\sum_{i=1}^{n-1}\bigl\{ (\boldsymbol{v}_i^\top \boldsymbol{x} - \lambda)_+ + (-\boldsymbol{v}_i^\top \boldsymbol{x} - \lambda)_+\bigr\} > b_2\biggr\} = \mathbbm{1}\{\|\mathcal{C}(\boldsymbol{x})\|_\infty > \lambda\} = h_{\lambda}^{\mathrm{CUSUM}}(\boldsymbol{x}),
	    \]
	    as desired.
	
	\subsection{The Proof of~\Cref{lem:generaltest}}\label{subsec:proofofgeneral LRtest}%
	    As $\boldsymbol{\Gamma}$ is invertible, (\ref{eq:generalCP}) in main text is equivalent to
	    \[
	        \boldsymbol{\Gamma}^{-1}\boldsymbol{X} = \boldsymbol{\Gamma}^{-1}\boldsymbol{Z}\boldsymbol{\beta}+\boldsymbol{\Gamma}^{-1}\boldsymbol{c}_{\tau} \phi + \boldsymbol{\xi}.
	    \]
	    Write $\tilde{\boldsymbol{X}}=\boldsymbol{\Gamma}^{-1}\boldsymbol{X}$, $\tilde{\boldsymbol{Z}}=\boldsymbol{\Gamma}^{-1}\boldsymbol{Z}$ and $\tilde{\boldsymbol{c}}_{\tau}=\boldsymbol{\Gamma}^{-1}\boldsymbol{c}_{\tau}$. If $\tilde{\boldsymbol{c}}_{\tau}$  lies in the column span of $\tilde{\boldsymbol{Z}}$, then the model with a change at $\tau$ is equivalent to the model with no change, and the likelihood-ratio test statistic will be 0. 
	    Otherwise we can assume, without loss of generality that $\tilde{\boldsymbol{c}}_{\tau}$ is orthogonal to each column of $\tilde{\boldsymbol{Z}}$: if this is not the case we can construct an equivalent model where we replace $\tilde{\boldsymbol{c}}_{\tau}$ with its projection to the space that is orthogonal to the column span of $\tilde{\boldsymbol{Z}}$.
	    
	    As $\boldsymbol{\xi}$ is a vector of independent standard normal random variables, the likelihood-ratio statistic for a change at $\tau$ against no change is a monotone function of the reduction in the residual sum of squares of the model with a change at $\tau$. The residual sum of squares of the no change model is
	    \[
	        \tilde{\boldsymbol{X}}^\top\tilde{\boldsymbol{X}} - \tilde{\boldsymbol{X}}^\top \tilde{\boldsymbol{Z}} (\tilde{\boldsymbol{Z}}^\top\tilde{\boldsymbol{Z}})^{-1} \tilde{\boldsymbol{Z}}^\top \tilde{\boldsymbol{X}}.
	    \]
	    The residual sum of squares for the model with a change at $\tau$ is
	    \[
	        \tilde{\boldsymbol{X}}^\top\tilde{\boldsymbol{X}} - \tilde{\boldsymbol{X}}^\top [\tilde{\boldsymbol{Z}}, \tilde{\boldsymbol{c}}_{\tau}]  ([\tilde{\boldsymbol{Z}},\tilde{\boldsymbol{c}}_{\tau}]^\top [\tilde{\boldsymbol{Z}},\tilde{\boldsymbol{c}}_{\tau}])^{-1} [\tilde{\boldsymbol{Z}},\tilde{\boldsymbol{c}}_{\tau}]^\top \tilde{\boldsymbol{X}} =
	        \tilde{\boldsymbol{X}}^\top\tilde{\boldsymbol{X}} - \tilde{\boldsymbol{X}}^\top \tilde{\boldsymbol{Z}} (\tilde{\boldsymbol{Z}}^\top\tilde{\boldsymbol{Z}})^{-1} \tilde{\boldsymbol{Z}}^\top \tilde{\boldsymbol{X}} - \tilde{\boldsymbol{X}}^\top \tilde{\boldsymbol{c}}_{\tau} (\tilde{\boldsymbol{c}}_{\tau}^\top\tilde{\boldsymbol{c}}_{\tau})^{-1} \tilde{\boldsymbol{c}}_{\tau}^\top \tilde{\boldsymbol{X}}.
	    \]
	    Thus, the reduction in residual sum of square of the model with the change at $\tau$ over the no change model is
	    \[
	        \tilde{\boldsymbol{X}}^\top \tilde{\boldsymbol{c}}_{\tau} (\tilde{\boldsymbol{c}}_{\tau}^\top\tilde{\boldsymbol{c}}_{\tau})^{-1} \tilde{\boldsymbol{c}}_{\tau}^\top \tilde{\boldsymbol{X}} =
	        \left( \frac{1}{\sqrt{\tilde{\boldsymbol{c}}_{\tau}^\top\tilde{\boldsymbol{c}}_{\tau}}}  \tilde{\boldsymbol{c}}_{\tau}^\top \tilde{\boldsymbol{X}} \right)^2
	    \]
	    Thus if we define
	    \[
	        \boldsymbol{v}_\tau=\frac{1}{\sqrt{\tilde{\boldsymbol{c}}_{\tau}^\top\tilde{\boldsymbol{c}}_{\tau}}}\tilde{\boldsymbol{c}}_{\tau}^{\top} \boldsymbol{\Gamma}^{-1,}
	    \]
	    then the likelihood-ratio test statistic is a monotone function of $|\boldsymbol{v}_\tau \boldsymbol{X}|$. This is true for all $\tau$ so the likelihood-ratio test is equivalent to
	    \[
	        \max_{\tau\in[n-1]} |\boldsymbol{v}_\tau \boldsymbol{X}| > \lambda,
	    \]
	    for some $\lambda$. This is of a similar form to the standard CUSUM test, except that the form of $\boldsymbol{v}_\tau$ is different. Thus, by the same argument as for~\Cref{lem:CUSUMinNNet} in main text, we can replicate this test with $h(\boldsymbol{x})\in \mathcal{H}_{1,2n-2}$, but with different weights to represent the different form for $\boldsymbol{v}_\tau$.
	\subsection{The Proof of~\Cref{lem:hypothesis_test}}\label{subsec:The_Proof_of_Lemma_lem:hypothesis_test}
	    \begin{proof}
	        (a) For each $i\in[n-1]$, since \( {\|\boldsymbol{v}_{i}\|_2}=1 \), we have \( \boldsymbol{v}_i^\top \boldsymbol{X}\sim N(0,1) \).
	        Hence, by the Gaussian tail bound and a union bound,
	        \[
	            \mathbb{P}\Bigl\{\|\mathcal{C}(\boldsymbol{X})\|_\infty >t \Bigr\}\leq \sum_{i=1}^{n-1}\mathbb{P}\left(\left|\boldsymbol{v}_i^\top \boldsymbol{X}\right| > t\right)\leq n\exp(-t^{2}/2).
	        \]
	        The result follows by taking $t=\sqrt{2\log(n/\varepsilon)}$.
	        
	        (b) We write $\boldsymbol{X} = \boldsymbol{\mu} + \boldsymbol{Z}$, where $\boldsymbol{Z} \sim N_n(0, I_n)$. Since the CUSUM transformation is linear, we have $\mathcal{C}(\boldsymbol{X}) = \mathcal{C}(\boldsymbol{\mu}) + \mathcal{C}(\boldsymbol{Z})$. By part (a) there is an event $\Omega$ with probability at least $1-\varepsilon$ on which $\|\mathcal{C}(\boldsymbol{Z})\|_\infty \leq \sqrt{2\log(n/\varepsilon)}$. Moreover, we have $\|\mathcal{C}(\boldsymbol{\mu})\|_\infty = |\boldsymbol{v}_{\tau}^\top \boldsymbol{\mu}| = |\mu_{\mathrm{L}} - \mu_{\mathrm{R}}| \sqrt{n\eta(1-\eta)} $. Hence on $\Omega$, we have by the triangle inequality that
	        \[
	            \|\mathcal{C}(\boldsymbol{X})\|_\infty \geq \|\mathcal{C}(\boldsymbol{\mu})\|_\infty  - \|\mathcal{C}(\boldsymbol{Z})\|_\infty  \geq |\mu_{\mathrm{L}} - \mu_{\mathrm{R}}| \sqrt{n\eta(1-\eta)} - \sqrt{2\log(n/\varepsilon)} > \sqrt{2\log(n/\varepsilon)},
	        \]
	        as desired.
	    \end{proof}

	\subsection{The Proof of~\Cref{cor:Generalisation}}\label{subsec:The Proof of Corollary_cor:Generalisation}
	    \begin{proof}
	        From~\Cref{lem:hypothesis_test} in main text with $\varepsilon=ne^{-nB^2/8}$, we have
	        \[
	            \mathbb{P}(h_\lambda^{\mathrm{CUSUM}}(\boldsymbol{X}) \neq Y\mid \tau,\mu_{\mathrm{L}}, \mu_{\mathrm{R}}) \leq ne^{-nB^2/8},
	        \]
	        and the desired result follows by integrating over $\pi_0$.
	    \end{proof}
	\subsection{Auxiliary Lemma}
	    \begin{lemma}\label{lem:geometric_grid_bound}
	        Define \( T^{\prime}\coloneqq\{ t_{0}\in \mathbb{Z}^{+} :{\left\lvert t_{0}-\tau \right\rvert}\leq\min(\tau,n-\tau)/2 \} \), for any \( t_{0}\in T^{\prime} \), we have
	        \begin{equation*}
	            \min_{t_{0}\in T^{\prime}}|\boldsymbol{v}_{t_{0}}^{\top}\boldsymbol{\mu}|\geq\frac{\sqrt{3}}{3}|\mu_{\mathrm{L}} - \mu_{\mathrm{R}}|\sqrt{n\eta(1-\eta)}.
	        \end{equation*}
	    \end{lemma}
	    \begin{proof}
	        For simplicity, let \( \Delta\coloneqq |\mu_{\mathrm{L}} - \mu_{\mathrm{R}}| \), we  can  compute the CUSUM test statistics \( a_{i}=|\boldsymbol{v}_{i}^{\top}\boldsymbol{\mu}| \) as:
	        \begin{equation*}
	            a_i=\begin{cases}
	                \Delta \left(1-\eta \right) \sqrt{\frac{n i}{n -i}} & 1\leq i  \leq \tau  \\
	                \Delta \eta \sqrt{\frac{n \left(n -i \right)}{i}} & \tau<i \leq n -1
	            \end{cases}
	        \end{equation*}
	        It is easy to verified that \( a_{\tau}\coloneqq\max_i(a_i)=\Delta\sqrt{n\eta(1-\eta)} \) when \( i=\tau \). Next, we only discuss the case of \( 1\leq \tau\leq \lfloor n/2 \rfloor \) as one can obtain the same result when \( \lceil n/2\rceil\leq \tau \leq n \) by the similar discussion.
	        
	        When \( 1\leq \tau\leq \lfloor n/2 \rfloor \), \( {\left\lvert t_{0} -\tau \right\rvert}\leq \min(\tau,n-\tau)/2 \) implies that \( t_{l} \leq t_{0}\leq t_{u} \) where \( t_{l}\coloneqq \lceil \tau/2 \rceil, t_{u}\coloneqq \lfloor 3\tau/2 \rfloor \). Because \( a_i \) is an increasing function of \( i \) on \( [1,\tau] \) and a decreasing function of \( i \) on \( [\tau+1,n-1] \) respectively, the minimum of \( a_{t_{0}}, t_{l} \leq t_{0}\leq t_{u} \) happens at either \( t_{l} \) or \( t_{u} \). Hence, we have
	        \begin{align*}
	            a_{t_{l}} & \geq a_{\tau/2}=a_{\tau}\sqrt{\frac{n-\tau}{2n-\tau}}\\
	            a_{t_{u}} & \geq a_{3\tau/2}=a_{\tau}\sqrt{\frac{2n-3\tau}{3(n-\tau)}}
	        \end{align*}
	        Define \( f(x)\coloneqq \sqrt{\frac{n-x}{2n-x}} \) and \( g(x)\coloneqq \sqrt{\frac{2n-3x}{3(n-x)}}\). We notice that \( f(x) \) and \( g(x) \) are both decreasing functions of \(x\in[1,n]\), therefore \( f(\lfloor n/2 \rfloor)\geq f(n/2)=\sqrt{3}/3 \) and \( g(\lfloor n/2 \rfloor)\geq g(n/2)=\sqrt{3}/3 \) as desired.
	    \end{proof}
	\subsection{The Proof of~\Cref{thm:VC_Generalisation_bound_old}}\label{subsec:The_Proof_of_Theorem_3.3}
	    \begin{proof}
	        Given any $L\geq 1$ and $\boldsymbol{m} = (m_1,\ldots,m_L)^\top$, let $m_0 := n$ and $m_{L+1}:=1$ and set $W^* = \sum_{r=1}^{L+1} m_{r-1}m_r$. Let $d\coloneqq \mathrm{VCdim}(\mathcal{H}_{L,\boldsymbol{m}})$, then by~\citet[Theorem~7]{bartlettNearlytightVCdimensionPseudodimension2019}, we have $d=O(LW^{*}\log(W^{*}))$.
	        Thus, by~\citet[Corollary~3.4]{mohriFoundationsMachineLearning2012}, for some universal constant $C>0$, we have with probability at least $1-\delta$ that
	        \begin{equation}
	            \mathbb{P}(h_{\mathrm{ERM}}(\boldsymbol{X})\neq Y\mid \mathcal{D}) \leq \min_{h\in \mathcal{H}_{L,\boldsymbol{m}}}\mathbb{P}(h(\boldsymbol{X})\neq Y)+\sqrt{\frac{8d\log(2eN/d)+8\log(4/\delta)}{N}}.
	            \label{Eq:BartlettMohri}
	        \end{equation}
	        Here, we have $L=1$, $m=2n-2$, $W^* = O(n^2)$, so $d = O(n^2\log(n))$. In addition, since $h^{\mathrm{CUSUM}}_\lambda \in \mathcal{H}_{1,2n-2}$, we have $\min_{h\in \mathcal{H}_{L,\boldsymbol{m}}} \leq \mathbb{P}(h^{\mathrm{CUSUM}}_\lambda(\boldsymbol{X})\neq Y) \leq ne^{-nB^2/8}$. Substituting these bounds into~\eqref{Eq:BartlettMohri} we arrive at the desired result.
	    \end{proof}
	
	\subsection{The Proof of~\Cref{thm:VC_Generalisation_bound_new}}\label{subsec:The_Proof_of_Corollary_3.1}
	    
	    The following lemma, gives the misclassification for the generalised CUSUM test where we only test for changes on a grid of $O(\log n)$ values.
	    
	    \begin{lemma}\label{lem:hypothesis_test_refine}
	        Fix $\varepsilon \in (0,1)$ and suppose that $\boldsymbol{X}\sim P(n,\tau,\mu_{\mathrm{L}},\mu_{\mathrm{R}})$ for some $\tau\in [n-1]$ and $\mu_{\mathrm{L}},\mu_{\mathrm{R}}\in\mathbb{R}$.
	        \begin{enumerate}
	            \item[(a)] If $\mu_{\mathrm{L}} = \mu_{\mathrm{R}}$, then
	            \[
	                \mathbb{P}\Bigl\{\max_{t\in T_0} |\boldsymbol{v}_t^\top \boldsymbol{X}| > \sqrt{2\log(|T_0|/\varepsilon)}\Bigr\}\leq\varepsilon.
	            \]
	            
	            \item[(b)] If $|\mu_{\mathrm{L}} - \mu_{\mathrm{R}}|\sqrt{\eta( 1-\eta)} > \sqrt{24\log(|T_0|/\varepsilon)/n}$, then we have
	            \[
	                \mathbb{P}\Bigl\{\max_{t\in T_0} |\boldsymbol{v}_t^\top \boldsymbol{X}| \leq  \sqrt{2\log(|T_0|/\varepsilon)}\Bigr\}\leq\varepsilon.
	            \]
	        \end{enumerate}
	    \end{lemma}
	    
	    \begin{proof}
	        (a) For each $t\in[n-1]$, since \( {\|\boldsymbol{v}_{t}\|_2}=1 \), we have \( \boldsymbol{v}_t^\top \boldsymbol{X}\sim N(0,1) \).
	        Hence, by the Gaussian tail bound and a union bound,
	        \[
	            \mathbb{P}\Bigl\{\max_{t\in T_0} |\boldsymbol{v}_t^\top \boldsymbol{X}| > y \Bigr\}\leq \sum_{t\in T_0}\mathbb{P}\left(\left|\boldsymbol{v}_t^\top \boldsymbol{X}\right| > y\right)\leq |T_{0}|\exp(-y^{2}/2).
	        \]
	        The result follows by taking $y=\sqrt{2\log(|T_{0}|/\varepsilon)}$.
	        
	        (b) There exists some $t_0 \in T_0$ such that $|t_0 - \tau| \leq \min\{\tau, n-\tau\}/2$. By~\Cref{lem:geometric_grid_bound}, we have
	        \[
	            |\boldsymbol{v}_{t_0}^\top \mathbb{E}\boldsymbol{X}| \geq \frac{\sqrt{3}}{3}\|\mathcal{C}( \mathbb{E}\boldsymbol{X})\|_\infty \geq \frac{\sqrt{3}}{3}|\mu_{\mathrm{L}}-\mu_{\mathrm{R}}|\sqrt{n\eta(1-\eta)} \geq 2 \sqrt{2\log(|T_0|/\varepsilon)}.
	        \]
	        Consequently, by the triangle inequality and result from part (a), we have with probability at least $1-\varepsilon$ that
	        \[
	            \max_{t\in T_0} |\boldsymbol{v}_{t}^\top \boldsymbol{X}| \geq |\boldsymbol{v}_{t_0}^\top \boldsymbol{X}| \geq |\boldsymbol{v}_{t_0}^\top \mathbb{E}\boldsymbol{X}| - |\boldsymbol{v}_{t_0}^\top (\boldsymbol{X}-\mathbb{E}\boldsymbol{X})| \geq \sqrt{2\log(|T_0|/\varepsilon)},
	        \]
	        as desired.
	    \end{proof}

	    Using the above lemma we have the following result.
	    
	    \begin{corollary}\label{cor:Generalisation_refine}
	        Fix $B>0$. Let $\pi_0$ be any prior distribution on $\Theta(B)$, then draw $(\tau, \mu_{\mathrm{L}}, \mu_{\mathrm{R}}) \sim \pi_0$, $\boldsymbol{X} \sim P(n, \tau, \mu_{\mathrm{L}}, \mu_{\mathrm{R}})$, and define $Y = \mathbbm{1}\{\mu_{\mathrm{L}} \neq \mu_{\mathrm{R}}\}$.  Then for $\lambda^{*} = B\sqrt{3n}/6$, the test $h^{\mathrm{CUSUM}_{*}}_{\lambda^{*}}$ satisfies
	        \[
	            \mathbb{P}(h^{\mathrm{CUSUM}_{*}}_{\lambda^{*}}(\boldsymbol{X}) \neq Y) \leq 2\lfloor\log_2(n)\rfloor e^{-nB^2/24}.
	        \]
	    \end{corollary}
	    \begin{proof}
	        Setting $\varepsilon = |T_{0}|e^{-nB^2/24}$ in~\Cref{lem:hypothesis_test_refine}, we have for any $(\tau,\mu_{\mathrm{L}}, \mu_{\mathrm{R}}) \in \Theta(B)$ that \[
	            \mathbb{P}(h^{\mathrm{CUSUM}_*}_{\lambda^*}(\boldsymbol{X}) \neq  \mathbbm{1}\{\mu_{\mathrm{L}}\neq \mu_{\mathrm{R}}\}) \leq |T_{0}|e^{-nB^2/24}.
	        \]
	        The result then follows by integrating over $\pi_0$ and the fact that $|T_{0}| = 2\lfloor \log_2(n)\rfloor$.
	    \end{proof}
	    \begin{proof}[Proof of~\Cref{thm:VC_Generalisation_bound_new}]
	        We follow the proof of~\Cref{thm:VC_Generalisation_bound_old} up to~\eqref{Eq:BartlettMohri}. From the conditions of the theorem, we have $W^* = O(Ln\log n)$. Moreover, we have $h^{\mathrm{CUSUM}_*}_{\lambda^*}\in \mathcal{H}_{1,4\lfloor\log_2(n)\rfloor} \subseteq \mathcal{H}_{L,\boldsymbol{m}}$. Thus,
	        \begin{align*}
	            \mathbb{P}(h_{\mathrm{ERM}}(\boldsymbol{X})\neq Y\mid \mathcal{D}) &\leq \mathbb{P}(h^{\mathrm{CUSUM}_*}_{\lambda^*}(\boldsymbol{X})\neq Y) + C\sqrt\frac{L^2n\log n\log(Ln)\log(N) + \log(1/\delta)}{N}\\
	            &\leq 2\lfloor \log_2(n)\rfloor e^{-nB^2/24} + C\sqrt\frac{L^2n\log^2(Ln)\log(N) + \log(1/\delta)}{N}
	        \end{align*}
	        as desired.
	    \end{proof}
	
	\subsection{Generalisation to time-dependent or heavy-tailed observations}
	    So far, for simplicity of exposition, we have primarily focused on change-point models with independent and identically distributed Gaussian observations. However, neural network based procedures can also be applied to time-dependent or heavy-tailed observations.  We first considered the case where the noise series $\xi_1,\ldots,\xi_n$ is a centred stationary Gaussian process with short-ranged temporal dependence. Specifically, writing $K(u):=\mathrm{cov}(\xi_t, \xi_{t+u})$, we assume that
	    \begin{equation}
	        \label{Eq:TemporalDependence}
	        \sum_{u=0}^{n-1} K(u) \leq D.
	    \end{equation}
	    
	    \begin{theorem}\label{Thm:Temporal}
	        Fix $B>0$, $n>0$ and let $\pi_0$ be any prior distribution on $\Theta(B)$. We draw $(\tau,\mu_{\mathrm{L}},\mu_{\mathrm{R}}) \sim \pi_0$, set $Y:=\mathbbm{1}\{\mu_{\mathrm{L}}\neq \mu_{\mathrm{R}}\}$ and generate $\boldsymbol{X} := \boldsymbol{\mu} + \boldsymbol{\xi}$ such that $\boldsymbol{\mu}:=(\mu_{\mathrm{L}}\mathbbm{1}\{i\leq \tau\}+\mu_{\mathrm{R}}\mathbbm{1}\{i > \tau\})_{i\in[n]}$ and  $\boldsymbol{\xi}$ is a centred stationary Gaussian process satisfying~\eqref{Eq:TemporalDependence}. Suppose that the training data $\mathcal{D} := \bigl((\boldsymbol{X}^{(1)},Y^{(1)}), \ldots, (\boldsymbol{X}^{(N)},Y^{(N)})\bigr)$ consist of independent copies of $(\boldsymbol{X},Y)$ and let $h_{\mathrm{ERM}}:=\argmin_{h\in\mathcal{L}_{L,\boldsymbol{m}}} L_N(h)$ be the empirical risk minimiser for a neural network with $L\geq 1$ layers and $\boldsymbol{m}= (m_1,\ldots,m_{L})^\top$ hidden layer widths. If $m_1\geq 4\lfloor \log_2(n)\rfloor$ and $m_{r}m_{r+1} = O(n\log n)$ for all $r\in[L-1]$, then for any $\delta\in(0,1)$, we have with probability at least $1-\delta$ that
	        \[
	            \mathbb{P}(h_{\mathrm{ERM}}(\boldsymbol{X}) \neq Y \mid \mathcal{D}) \leq 2\lfloor \log_2(n)\rfloor e^{-nB^2/(48D)} + C\sqrt\frac{L^2 n\log^2(Ln)\log(N)+\log(1/\delta)}{N}.
	        \]
	    \end{theorem}
	    \begin{proof}
	        By the proof of \citet[supplementary Lemma 10]{wangHighDimensionalChange2018},
	        \[
	            \mathbb{P}\bigl\{\max_{t\in T_0} |\boldsymbol{v}_t^\top \boldsymbol{\xi}| > B\sqrt{3n}/6\bigr\} \leq |T_0|e^{-nB^2/(48D)}.
	        \]
	        On the other hand, for $t_0$ defined in the proof of~\Cref{lem:geometric_grid_bound}, we have that  $|\mu_{\mathrm{L}}-\mu_{\mathrm{R}}|\sqrt{\tau(n-\tau)}/n > B$, then $|\boldsymbol{v}_{t_0}^\top \mathbb{E} X| \geq B\sqrt{3n}/3$. Hence for $\lambda^* = B\sqrt{3n}/6$, we have $h_{\lambda^*}^{\mathrm{CUSUM}_*}$ satisfying
	        \[
	            \mathbb{P}(h_{\lambda^*}^{\mathrm{CUSUM}_*}(\boldsymbol{X}\neq Y)) \leq |T_0|e^{-nB^2/(48D)}.
	        \]
	        We can then complete the proof using the same arguments as in the proof of~\Cref{thm:VC_Generalisation_bound_new}.
	    \end{proof}
	    
	    We now turn to non-Gaussian distributions and recall that the Orlicz $\psi_\alpha$-norm of a random variable $Y$ is defined as
	    \[
	        \|Y\|_{\psi_{\alpha}} := \inf\{\eta: \mathbb{E}\exp(|Y/\eta|^\alpha) \leq 2\}.
	    \]
	    For $\alpha \in (0,2)$, the random variable $Y$ has heavier tail than a sub-Gaussian distribution. The following lemma is a direct consequence of~\citet[Theorem~3.1]{kuchibhotla2022moving} (We state the version used in \citet[Proposition~14]{li2023robust}).
	    \begin{lemma}\label{Lemma:Kuchibhotla}
	        Fix $\alpha\in(0,2)$. Suppose $\boldsymbol{\xi}=(\xi_1,\ldots,\xi_n)^\top$ has independent components satisfying $\mathbb{E}\xi_t = 0$, $\mathrm{Var}(\xi_t)=1$ and $\|\xi_t\|_{\psi_\alpha} \leq K$ for all $t\in[n]$. There exists $c_\alpha>0$, depending only on $\alpha$, such that for any $1\leq t\leq n/2$, we have
	        \[
	            \mathbb{P}\bigl(|\boldsymbol{v}_t^\top \boldsymbol{\xi}| \geq y\bigr) \leq \exp\biggl\{1 - c_\alpha \min\biggl\{ \biggl(\frac{y}{K}\biggr)^2,\, \biggl(\frac{y}{K\|\boldsymbol{v}_t\|_{\beta(\alpha)}}\biggr)^\alpha\biggr\}\biggr\},
	        \]
	        where $\beta(\alpha) = \infty$ for $\alpha \leq 1$ and $\beta(\alpha) = \alpha/(\alpha-1)$ when $\alpha > 1$.
	    \end{lemma}

	    \begin{theorem}\label{Thm:SubWeibull}
	        Fix $\alpha\in(0,2)$, $B>0$, $n>0$ and let $\pi_0$ be any prior distribution on $\Theta(B)$. We draw $(\tau,\mu_{\mathrm{L}},\mu_{\mathrm{R}}) \sim \pi_0$, set $Y:=\mathbbm{1}\{\mu_{\mathrm{L}}\neq \mu_{\mathrm{R}}\}$ and generate $\boldsymbol{X} := \boldsymbol{\mu} + \boldsymbol{\xi}$ such that $\boldsymbol{\mu}:=(\mu_{\mathrm{L}}\mathbbm{1}\{i\leq \tau\}+\mu_{\mathrm{R}}\mathbbm{1}\{i > \tau\})_{i\in[n]}$ and $\boldsymbol{\xi} = (\xi_1,\ldots,\xi_n)^\top$ satisfies $\mathbb{E}\xi_i = 0$, $\mathrm{Var}(\xi_i) = 1$ and $\|\xi_i\|_{\psi_\alpha}\leq K$ for all $i\in[n]$. Suppose that the training data $\mathcal{D} := \bigl((\boldsymbol{X}^{(1)},Y^{(1)}), \ldots, (\boldsymbol{X}^{(N)},Y^{(N)})\bigr)$ consist of independent copies of $(\boldsymbol{X},Y)$ and let $h_{\mathrm{ERM}}:=\argmin_{h\in\mathcal{L}_{L,\boldsymbol{m}}} L_N(h)$ be the empirical risk minimiser for a neural network with $L\geq 1$ layers and $\boldsymbol{m}= (m_1,\ldots,m_{L})^\top$ hidden layer widths. If $m_1\geq 4\lfloor \log_2(n)\rfloor$ and $m_{r}m_{r+1} = O(n\log n)$ for all $r\in[L-1]$, then there exists a constant $c_\alpha > 0$, depending only on $\alpha$ such that for any $\delta\in(0,1)$, we have with probability at least $1-\delta$ that
	        \[
	            \mathbb{P}(h_{\mathrm{ERM}}(\boldsymbol{X}) \neq Y\mid\mathcal{D}) \leq 2\lfloor \log_2(n)\rfloor  e^{1-c_\alpha(\sqrt{n}B/K)^\alpha} + C\sqrt\frac{L^2 n\log^2(Ln)\log(N)+\log(1/\delta)}{N}.
	        \]
	    \end{theorem}
	    \begin{proof}
	        For $\alpha\in(0,2)$, we have $\beta(\alpha) > 2$, so $\|\boldsymbol{v}_t\|_{\beta(\alpha)} \geq \|\boldsymbol{v}_t\|_2 = 1$. Thus, from~\Cref{Lemma:Kuchibhotla}, we have $\mathbb{P}(|\boldsymbol{v}_t^\top \boldsymbol{\xi}| \geq y) \leq e^{1-c_\alpha(y/K)^\alpha}$. Thus, following the proof of~\Cref{cor:Generalisation_refine}, we can obtain that $\mathbb{P}(h_{\lambda^*}^{\mathrm{CUSUM}_*}(\boldsymbol{X}\neq Y)) \leq 2\lfloor\log_2(n)\rfloor e^{1-c_\alpha(\sqrt{n}B/K)^\alpha}$. Finally, the desired conclusion follows from the same argument as in the proof of~\Cref{thm:VC_Generalisation_bound_new}.
	    \end{proof}
	    %
	
	\subsection{Multiple change-point estimation}
	    \Cref{alg:change_point_detection} is a general scheme for turning a change-point classifier into a location estimator. While it is theoretically challenging to derive theoretical guarantees for the neural network based change-point location estimation error, we motivate this methodological proposal here by showing that~\Cref{alg:change_point_detection}, applied in conjunction with a CUSUM-based classifier have optimal rate of convergence for the change-point localisation task. We consider the model $x_i = \mu_i + \xi_i$, where $\xi_i\stackrel{\mathrm{iid}}{\sim} N(0,1)$ for $i\in[n^*]$. Moreover, for a sequence of change-points $0 = \tau_0 < \tau_1 < \cdots < \tau_\nu < n = \tau_{\nu+1}$ satisfying $\tau_r - \tau_{r-1} \geq 2n$ for all $r\in[\nu+1]$ we have $\mu_i = \mu^{(r-1)}$ for all $i\in[\tau_{r-1},\tau_r]$, $r\in[\nu+1]$.
	    \begin{theorem}\label{Thm:Localisation}
	        Suppose data $x_1,\ldots,x_{n^*}$ are generated as above satisfying $|\mu^{(r)} - \mu^{(r-1)}| > 2\sqrt{2}B$ for all $r\in[\nu]$. Let $h_{\lambda^*}^{\mathrm{CUSUM}_*}$ be defined as in~\Cref{cor:Generalisation_refine}. Let $\hat\tau_1,\ldots,\hat\tau_{\hat\nu}$ be the output of~\Cref{alg:change_point_detection} with input  $x_1,\ldots,x_{n^*}$, $\psi = h_{\lambda^*}^{\mathrm{CUSUM}_*}$ and $\gamma = \lfloor n/2\rfloor / n$. Then we have
	        \[
	            \mathbb{P}\biggl\{\hat\nu = \nu \text{ and } |\tau_i - \hat \tau_i| \leq \frac{2B^2}{|\mu^{(r)} - \mu^{(r-1)}|^2}\biggr\} \geq 1 - 2n^*\lfloor \log_2(n)\rfloor e^{-nB^2/24}.
	        \]
	    \end{theorem}
	    \begin{proof}
	        For simplicity of presentation, we focus on the case where $n$ is a multiple of 4, so $\gamma = 1/2$. Define
	        \begin{align*}
	            I_0 &:= \{i: \mu_{i+n-1} = \mu_i\},\\
	            I_1 & := \biggl\{i: |\mu_{i+n-1}-\mu_i|\max_{r\in[\nu]} \sqrt{\frac{(\tau_r-i)(i+n-\tau_r)}{n^2}} \geq B\biggr\}.
	        \end{align*}
	        By~\Cref{lem:hypothesis_test_refine} and a union bound, the event
	        \[
	        \Omega = \bigl\{h_{\lambda^*}^{\mathrm{CUSUM}_*}(\boldsymbol{X}^*_{[i,i+n)}) = k,\text{ for all $i\in I_k$, $k=0,1$}\bigr\}
	                \]
	                has probability at least $1 - 2n^*\lfloor \log_2(n)\rfloor e^{-nB^2/24}$. We work on the event $\Omega$ henceforth. Denote $\Delta_r := 2B^2/|\mu^{(r)}-\mu^{(r-1)}|^2$.  Since $|\mu^{(r)}-\mu^{(r-1)}| > 2\sqrt{2}B$, we have $\Delta_r< n/4$. Note that for each $r\in [\nu]$, we have $\{i:\tau_{r-1} < i\leq \tau_r -n \text{ or } \tau_{r} < i\leq \tau_{r+1} -n\} \subseteq I_0$ and $\{i : \tau_r-n+\Delta_r < i \leq \tau_r -\Delta_r\}\subseteq I_1$. Consequently, $\bar L_i$ defined in~\Cref{alg:change_point_detection} is below the threshold $\gamma = 1/2$ for all $i \in (\tau_{r-1}+n/2, \tau_r-n/2] \cup (\tau_r+n/2, \tau_{r+1}-n/2]$, monotonically increases for $i\in (\tau_r-n/2, \tau_r-\Delta]$ and monotonically decreases for $i\in (\tau_r+\Delta, \tau_r+n/2]$ and is above the threshold $\gamma$ for $i \in (\tau_r-\Delta,\tau_r+\Delta]$. Thus, exactly one change-point, say $\hat\tau_r$, will be identified on $(\tau_{r-1}+n/2, \tau_{r+1}-n/2]$ and $\hat\tau_r = \argmax_{i\in(\tau_{r-1}+n/2, \tau_{r+1}-n/2]} \bar L_i \in (\tau_r-\Delta, \tau_r+\Delta]$ as desired. Since the above holds for all $r\in[\nu]$, the proof is complete.
	    \end{proof}
	    
	    Assuming that $\log(n^*) \asymp \log(n)$ and choosing $B$ to be of order $\sqrt{\log n}$, the above theorem shows that using the CUSUM-based change-point classifier $\psi = h_{\lambda^*}^{\mathrm{CUSUM}_*}$ in conjunction with~\Cref{alg:change_point_detection} allows for consistent estimation of both the number of locations of multiple change-points in the data stream. In fact, the rate of estimating each change-point, $2B^2/|\mu^{(r)}-\mu^{(r-1)}|^2$, is minimax optimal up to logarithmic factors \citep[see, e.g.][Proposition~6]{verzelen2020optimal}. An inspection of the proof of~\Cref{Thm:Localisation} reveals that the same result would hold for any $\psi$ for which the event $\Omega$ holds with high probability. In view of the representability of $h_{\lambda^*}^{\mathrm{CUSUM}_*}$ in the class of neural networks, one would intuitively expect that a similar theoretical guarantee as in~\Cref{Thm:Localisation} would be available to the  empirical risk minimiser in the corresponding neural network function class. However, the particular way in which we handle the generalisation error in the proof of~\Cref{thm:VC_Generalisation_bound_new} makes it difficult to proceed in this way, due to the fact that the distribution of the data segments obtained via sliding windows have complex dependence and no longer follow a common prior distribution $\pi_{0}$ used in~\Cref{thm:VC_Generalisation_bound_old}.

\section{Simulation and Result}\label{sec:simulation_and_result}
	\subsection{Simulation for Multiple Change-types}
	    In this section, we illustrate the numerical study for one-change-point but with multiple change-types: change in mean, change in slope and change in variance.
	    
	    The data set with change/no-change in mean is generated from \( P(n,\tau,\mu_{\mathrm{L}},\mu_{\mathrm{R}}) \).
	    We employ the model of change in slope from~\citet{fearnheadDetectingChangesSlope2019}, namely
	    \begin{equation*}
	        x_{t}=f_t+\xi_{t}=\begin{cases}
	            \phi_{0}+\phi_{1}t+\xi_{t}&\quad \text{if } 1\leq t\leq \tau \\
	            \phi_{0}+(\phi_{1}-\phi_{2})\tau+\phi_{2} t+\xi_{t}&\quad \tau+1\leq t\leq n,
	        \end{cases}
	    \end{equation*}
	    where \( \phi_{0},\phi_{1} \) and \( \phi_{2} \) are parameters that can guarantee the continuity of two pieces of linear function at time \( t=\tau \). We use the following model to generate the data set with change in variance.
	    \begin{equation*}
	        y_{t}=\begin{cases}
	            \mu+\varepsilon_{t}\quad \varepsilon_{t}\sim N(0,\sigma_{1}^{2}), & \text{ if } t \leq \tau \\
	            \mu+\varepsilon_{t}\quad \varepsilon_{t}\sim N(0,\sigma_{2}^{2}), & \text{ otherwise }
	        \end{cases}
	    \end{equation*}
	    where \( \sigma_{1}^{2}, \sigma_{2}^{2} \) are the variances of two Gaussian distributions.\ \( \tau \) is the change-point in variance. When \( \sigma_{1}^{2}= \sigma_{2}^{2} \), there is no-change in model. The labels of no change-point, change in mean only, change in variance only, no-change in variance and change in slope only are 0, 1, 2, 3, 4 respectively. For each label, we randomly generate \( N_{sub} \) time series. In each replication of \( N_{sub} \), we update these parameters: \( \tau, \mu_{\mathrm{L}}, \mu_{\mathrm{R}}, \sigma_{1}, \sigma_{2},\alpha_{1}, \phi_{1},\phi_{2}\). To avoid the boundary effect, we randomly choose \( \tau \) from the discrete uniform distribution \( U(n^{\prime}+1,n-n^{\prime}) \) in each replication, where \( 1\leq n^{\prime}< \lfloor n/2 \rfloor, n^{\prime}\in \mathbb{N} \). The other parameters are generated as follows:
	    \begin{itemize}
	        \item \(  \mu_{\mathrm{L}}, \mu_{\mathrm{R}}\sim U(\mu_{l},\mu_{u}) \) and \( \mu_{dl}\leq \left\vert  \mu_{\mathrm{L}}-\mu_{\mathrm{R}} \right\vert\leq\mu_{du}\), where \( \mu_{l},\mu_{u} \) are the lower and upper bounds of  \( \mu_{\mathrm{L}},\mu_{\mathrm{R}} \).\ \( \mu_{dl},\mu_{du} \) are the lower and upper bounds of  \( \left\vert  \mu_{\mathrm{L}}-\mu_{\mathrm{R}} \right\vert \).
	        \item \(  \sigma_{1}, \sigma_{2}\sim U(\sigma_{l},\sigma_{u}) \) and \( \sigma_{dl}\leq \left\vert  \sigma_{1}-\sigma_{2} \right\vert\leq\sigma_{du}\), where \( \sigma_{l},\sigma_{u}\) are the lower and upper bounds of  \( \sigma_{1},\sigma_{2} \).\ \( \sigma_{dl},\sigma_{du} \) are the lower and upper bounds of  \( \left\vert  \sigma_{1}-\sigma_{2} \right\vert \).
	        \item \(  \phi_{1}, \phi_{2}\sim U(\phi_{l},\phi_{u}) \) and \( \phi_{dl}\leq \left\vert  \phi_{1}-\phi_{2} \right\vert\leq\phi_{du}\), where \( \phi_{l},\phi_{u}\) are the lower and upper bounds of  \( \phi_{1},\phi_{2} \).\ \( \phi_{dl},\phi_{du} \) are the lower and upper bounds of  \( \left\vert  \phi_{1}-\phi_{2} \right\vert \).
	    \end{itemize}
	    Besides, we let \(\mu=0\), \( \phi_{0}=0\) and the noise follows normal distribution with mean 0. For flexibility, we let the noise variance of change in mean and slope be \(  0.49 \) and \( 0.25 \) respectively. Both Scenarios 1 and 2 defined below use the neural network architecture displayed in~\Cref{fig:Arch-ResNet.eps}.
	    
	    \textbf{Benchmark}.~\citet{amini2017} reviewed the methodologies for change-point detection in different types. To be simple, we employ the Narrowest-Over-Threshold (NOT)~\citep{baranowskiNarrowestoverthresholdDetectionMultiple2019} and single variance change-point detection~\citep{chen2012parametric} algorithms to detect the change in mean, slope and variance respectively. These two algorithms are available in \textbf{R} packages:~\href{https://CRAN.R-project.org/package=not}{\textbf{not}} and~\href{https://CRAN.R-project.org/package=changepoint}{\textbf{changepoint}}. The oracle likelihood based tests $\text{LR}^{\mathrm{oracle}}$ means that we pre-specified whether we are testing for change in mean, variance or slope. For the construction of adaptive likelihood-ratio based test $\text{LR}^{\mathrm{adapt}}$, we first separately apply 3 detection algorithms of change in mean, variance and slope to each time series, then we can compute 3 values of Bayesian information criterion (BIC) for each change-type based on the results of change-point detection. Lastly, the corresponding label of minimum of BIC values is treated as the predicted label.
	    
	    \textbf{Scenario 1: Weak SNR}.
	    Let \( n=400 \), \( N_{sub}=2000 \) and \( n^{\prime}=40\). The data is generated by the parameters settings in~\Cref{tab:The_Parameters_for_Weak_and_Strong_Signals}. We use the model architecture in~\Cref{fig:Arch-ResNet.eps} to train the classifier. The learning rate is 0.001, the batch size is 64, filter size in convolution layer is 16, the kernel size is \( (3, 30) \), the epoch size is 500. The transformations are  (\( x, x^2\)). We also use the inverse time decay technique to dynamically reduce the learning rate. The result which is displayed in~\Cref{tab:The accuracy of LR and NN} of main text shows that the test accuracy of  $\text{LR}^{\mathrm{oracle}}$, $\text{LR}^{\mathrm{adapt}}$ and $\text{NN}$  based on 2500 test data sets are  0.9056, 0.8796 and 0.8660 respectively.
	    \begin{table}
	        \caption{\label{tab:The_Parameters_for_Weak_and_Strong_Signals}The parameters for weak and strong signal-to-noise ratio (SNR).}
	        \centering
	        \fbox{%
	            \begin{tabular}{*{5}c}
	                Chang in mean & \( \mu_{l} \) & \( \mu_{u} \)& \( \mu_{dl} \)& \( \mu_{du} \) \\ \hline
	                Weak SNR & -5 & 5 & 0.25 & 0.5\\
	                Strong SNR  & -5 & 5 & 0.6 & 1.2\\ \hline
	                Chang in variance & \( \sigma_{l} \) & \( \sigma_{u} \)& \( \sigma_{dl} \)& \( \sigma_{du} \) \\ \hline
	                Weak SNR & 0.3 & 0.7 & 0.12 & 0.24\\
	                Strong SNR  & 0.3 & 0.7 & 0.2 & 0.4 \\ \hline
	                Change in slope & \( \phi_{l} \) & \( \phi_{u} \)& \( \phi_{dl} \)& \( \phi_{du} \) \\ \hline
	                Weak SNR & -0.025 & 0.025 & 0.006 & 0.012\\
	                Strong SNR  & -0.025 & 0.025 &  0.015 & 0.03 \\ \hline
	            \end{tabular}}
	    \end{table}
	    
	    \textbf{Scenario 2: Strong SNR}. The parameters for generating strong-signal data are listed in~\Cref{tab:The_Parameters_for_Weak_and_Strong_Signals}. The other hyperparameters are same as in Scenario 1.\ The test accuracy of  $\text{LR}^{\mathrm{oracle}}$, $\text{LR}^{\mathrm{adapt}}$ and $\text{NN}$  based on 2500 test data sets are  0.9924, 0.9260 and 0.9672 respectively. We can see that the neural network-based approach achieves higher classification accuracy than the adaptive likelihood based method.
	\subsection{Some Additional Simulations}\label{subsec:Some_Additional_Simulations}
	    
	    \medskip
	    \subsubsection{Simulation for simultaneous changes}\label{subsubsec:Simulation_for_simultaneous_changes}
	        
	        In this simulation, we compare the classification accuracies of likelihood-based classifier and NN-based classifier under the circumstance of simultaneous changes. For simplicity, we only focus on two classes: no change-point (Class 1) and change in mean and variance at a same change-point (Class 2). The change-point location \( \tau \) is randomly drawn from \(\mathrm{Unif}\{40,\ldots, n-41\} \) where \( n = 400 \) is the length of time series. Given \( \tau \), to generate the data of Class 2, we use the parameter settings of change in mean and change in variance in~\Cref{tab:The_Parameters_for_Weak_and_Strong_Signals} to randomly draw \( \mu_{\mathrm{L}}, \mu_{\mathrm{R}} \) and \( \sigma_{1}, \sigma_{2} \) respectively. The data before and after the change-point \( \tau \) are generated from \( N(\mu_{\mathrm{L}},\sigma_{1}^{2}) \) and \( N(\mu_{\mathrm{R}},\sigma_{2}^{2}) \) respectively. To generate the data of Class 1, we just draw the data from \( N(\mu_{\mathrm{L}},\sigma_{1}^{2}) \). Then, we repeat each data generation of Class 1 and 2 \( 2500 \) times as the training dataset. The test dataset is generated in the same procedure as the training dataset, but the testing size is 15000. We use two classifiers: likelihood-ratio (LR) based  classifier~\citep[][p.59]{chen2012parametric} and a 21-residual-block neural network (NN) based classifier displayed in~\Cref{fig:Arch-ResNet.eps} to evaluate the classification accuracy of simultaneous change v.s.\ no change. The result are displayed in~\Cref{tab:The accuracy of LR and NN revision}. We can see that under weak SNR, the NN has a good performance than LR-based method while it performs as well as the LR-based method under strong SNR.\
	        \begin{table}
	            \caption{\label{tab:The accuracy of LR and NN revision} Test classification accuracy of likelihood-ratio (LR) based classifier~\citep[][p.59]{chen2012parametric} and our residual neural network (NN) based classifier with 21 residual blocks for setups with weak and strong signal-to-noise ratios (SNR). Data are generated as a mixture of no change-point (Class 1), change in mean and variance at a same change-point (Class 2). We report the true positive rate of each class and the accuracy in the last row.  The optimal threshold value of LR is chosen by the grid search method on the training dataset.}
	            \centering
	            \fbox{
	                \begin{tabular}{*{5}c}
	                    \multirow{2}{*}{}&  \multicolumn{2}{c}{Weak SNR} &  \multicolumn{2}{c}{Strong SNR}\\
	                    \cmidrule(lr){2-3}\cmidrule(lr){4-5}
	                    & LR  & NN &  LR & NN \\
	                    \hline
	                    Class 1 & 0.9823 & 0.9668 &  1.0000 & 0.9991 \\
	                    Class 2	&   0.8759 & 0.9621 &   0.9995 & 0.9992\\
	                    Accuracy & 	 0.9291 & 0.9645 & 0.9997 & 0.9991
	                \end{tabular}}
	        \end{table}
	    \subsubsection{Simulation for heavy-tailed noise}\label{subsubsec:Simulation_for_heavy_tail_noise}
	        In this simulation, we compare the performance of Wilcoxon change-point test~\citep{dehling2013changepoint}, CUSUM, simple neural network $\mathcal{H}_{L,\boldsymbol{m}}$ as well as truncated $\mathcal{H}_{L,\boldsymbol{m}}$ for heavy-tailed noise. Consider the model: \( X_{i}=\mu_{i}+\xi_{i},\quad i\geq 1, \) where \( (\mu_{i})_{i\geq 1} \) are signals and \( (\xi_{i})_{i\geq 1} \) is a stochastic process. To test the null hypothesis
	        \begin{equation*}
	            \mathbb{H}: \mu_{1}=\mu_{2}=\cdots=\mu_{n}
	        \end{equation*} against the alternative
	        \begin{equation*}
	            \mathbb{A}:~\text{There exists } 1\leq k \leq n-1~\text{such that } \mu_{1}=\cdots=\mu_{k}\neq \mu_{k+1}=\cdots=\mu_{n}.
	        \end{equation*}
	        \citet{dehling2013changepoint} proposed the so-called Wilcoxon type of cumulative sum statistic
	        \begin{equation}\label{eq:wilcoxon_sum}
	            T_{n}\coloneqq \max_{1\leq k<n}{\left\lvert\frac{2\sqrt{k(n-k)}}{n} \frac{1}{n^{3/2}}\sum_{i=1}^{k}\sum_{j=k+1}^{n}\left( \mathbf{1}_{\{X_{i}<X_{j}\}} -1/2\right) \right\rvert}
	        \end{equation}
	        to detect the change-point in time series with outlier or heavy tails. Under the null hypothesis \( \mathbb{H} \), the limit distribution of \( T_{n} \)\footnote{The definition of \( T_{n} \) in~\citet[Theorem 3.1]{dehling2013changepoint} does not include \( 2\sqrt{k(n-k)}/n \). However, the repository of the \textbf{R} package \textbf{robts}~\citep{R:robts} normalises the Wilcoxon test by this item, for details see function \textbf{wilcoxsuk} in~\href{https://r-forge.r-project.org/scm/viewvc.php/pkg/src/wilcox.c?view=markup\&root=robts}{here}. In this simulation, we adopt the definition of~\eqref{eq:wilcoxon_sum}.} can be approximately by the supreme of standard Brownian bridge process \( (W^{(0)}(\lambda))_{0\leq\lambda\leq 1} \) up to a scaling factor~\citep[][Theorem 3.1]{dehling2013changepoint}. In our simulation, we choose the optimal thresh value based on the training dataset by using the grid search method.
	        
	        The truncated simple neural network means that we truncate the data by the \( z \)-score in data preprocessing step, i.e.\ given vector \( \boldsymbol{x}= ( x_{1},x_{2},\ldots,x_{n} )^{\top} \), then \( x_{i}[{\left\lvert x_{i}-\bar{x} \right\rvert}>Z\sigma_{x}]=\bar{x}+\text{sgn}(x_{i}-\bar{x})Z\sigma_{x} \), \( \bar{x} \) and \( \sigma_{x} \) are the mean and standard deviation of \( \boldsymbol{x} \).
	        
	        The training dataset is generated by using the same parameter settings of~\Cref{fig:simulation_result}(d) of the main text. The result of misclassification error rate (MER) of each method is reported in~\Cref{fig:S3WilcoxonTruncation.eps}. We can see that truncated simple neural network has the best performance. As expected, the Wilcoxon based test has better performance than the simple neural network based tests. However, we would like to mention that the main focus of~\Cref{fig:simulation_result} of the main text is to demonstrate the point that simple neural networks can replicate the performance of CUSUM tests. Even though, the prior information of heavy-tailed noise is available, we still encourage the practitioner to use simple neural network by adding the \( z \)-score truncation in data preprocessing step.
	        \begin{figure}[htbp]
	            \centering
	            \includegraphics[width=0.7\textwidth]{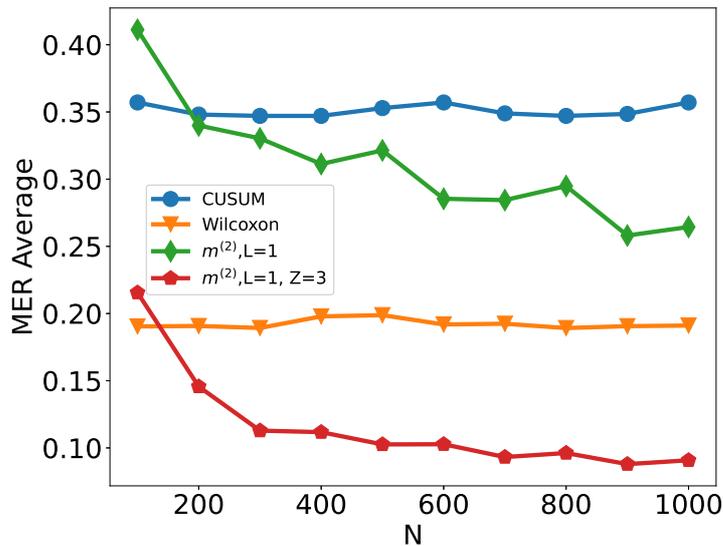}
	            \caption{Scenario S3 with Cauchy noise by adding Wilcoxon type of change-point detection method~\citep{dehling2013changepoint} and simple neural network with truncation in data preprocessing. The average misclassification error rate (MER) is computed on a test set of size $N_{\mathrm{test}}=15000$,  against training sample size $N$ for detecting the existence of a change-point on data series of length $n=100$. We compare the performance of the CUSUM test, Wilcoxon test, $\mathcal{H}_{1,m^{(2)}}$ and $\mathcal{H}_{1,m^{(2)}}$ with \( Z=3 \) where $m^{(2)} = 2n-2$ and \( Z=3 \) means the truncated \(z\)-score, i.e.\ given vector \( \protect \boldsymbol{x}= ( x_{1},x_{2},\ldots,x_{n} )^{\top} \), then \( x_{i}[{\left|x_{i}-\bar{x} \right|}>Z\sigma_{x}]=\bar{x}+\mathrm{sgn}(x_{i}-\bar{x})Z\sigma_{x} \), \( \bar{x} \) and \( \sigma_{x} \) are the mean and standard deviation of \( \protect \boldsymbol{x} \).}\label{fig:S3WilcoxonTruncation.eps}
	        \end{figure}
	    \subsubsection{Robustness Study}\label{subsubsec:Robustness_Study}
	        This simulation is an extension of numerical study of~\Cref{sec:Simulation_Study} in main text. We trained our neural network using training data generated under scenario S1 with $\rho_t=0$ (i.e.\ corresponding to~\Cref{fig:simulation_result}(a) of the main text), but generate the test data under settings corresponding to~\Cref{fig:simulation_result}(a, b, c, d). In other words, apart the top-left panel, in the remaining panels of~\Cref{fig:simulation_resultA2Others}, the trained network is misspecified for the test data. We see that the neural networks continue to work well in all panels, and in fact have performance similar to those in~\Cref{fig:simulation_result}(b, c, d) of the main text. This indicates that the trained neural network has likely learned features related to the change-point rather than any distributional specific artefacts.
	        \begin{figure}[htbp]
	            \begin{minipage}{1\linewidth}
	                \makebox[.5\linewidth]{\includegraphics[width=.45\linewidth, height=0.36\textwidth]{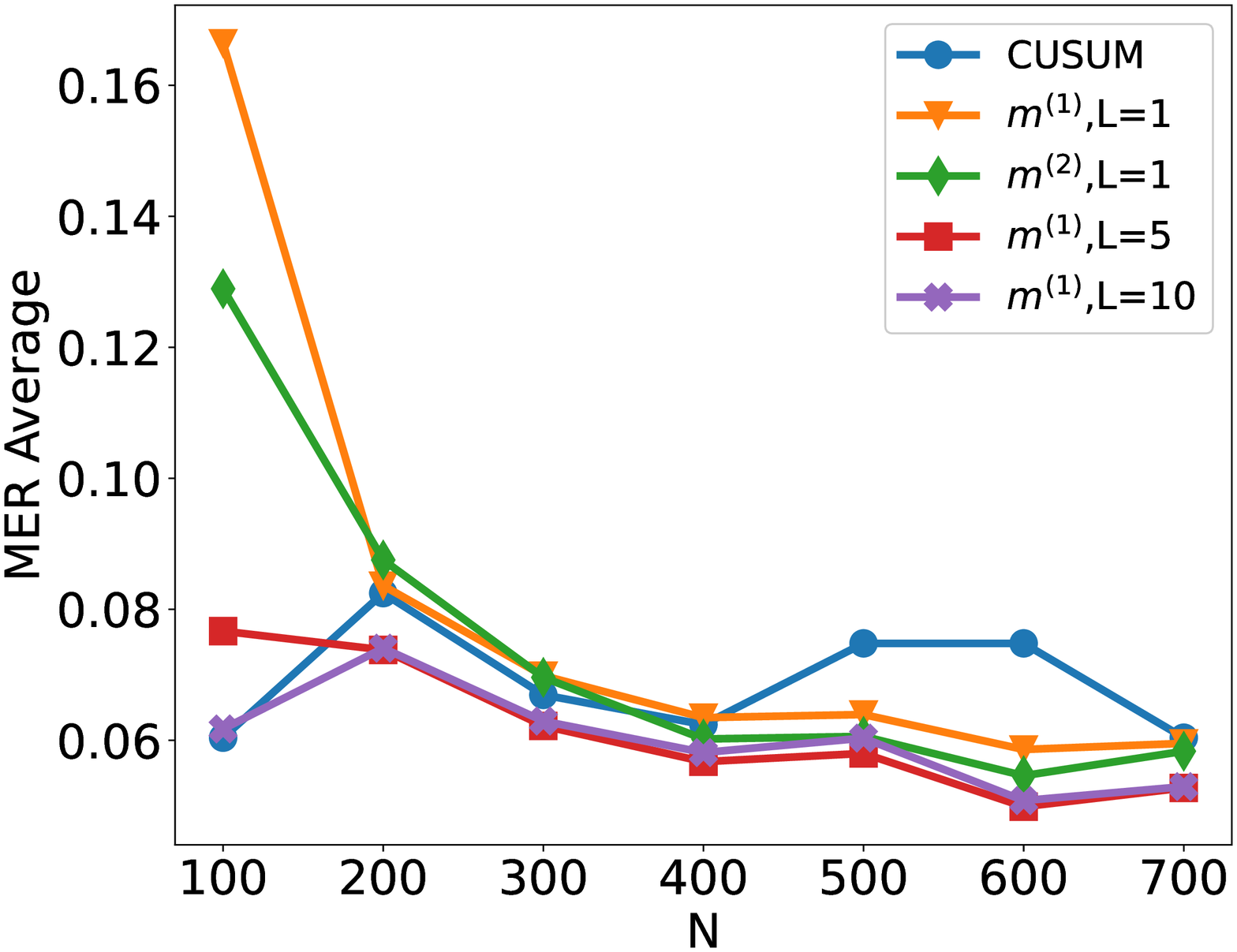}}%
	                \makebox[.5\linewidth]{\includegraphics[width=.45\linewidth, height=0.36\textwidth]{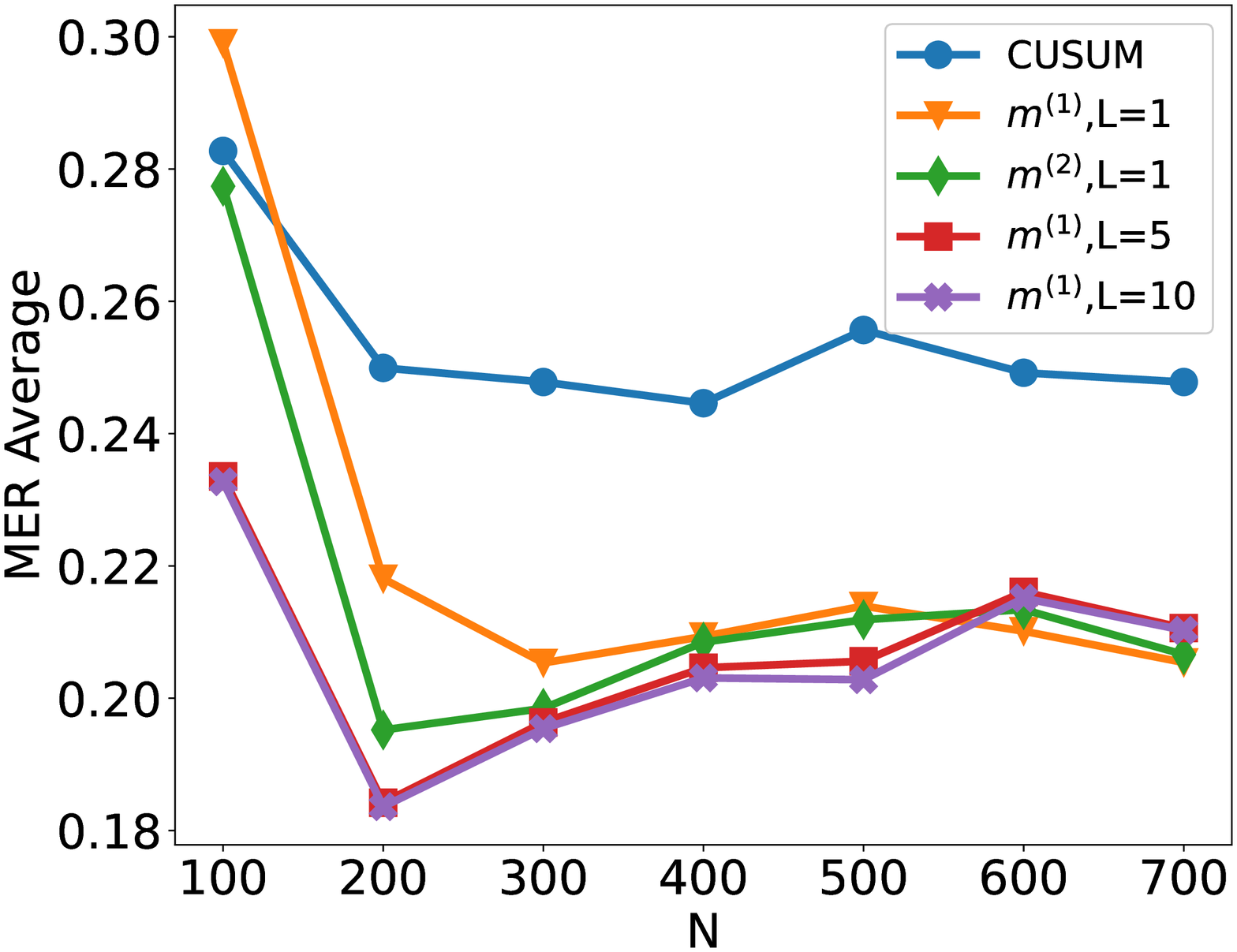}}
	                \makebox[.5\linewidth]{\small (a) Trained S1 (\( \rho_{t}= 0 \)) \( \to \) S1 (\( \rho_{t}= 0 \))}%
	                \makebox[.5\linewidth]{\small (b)Trained S1 (\( \rho_{t}= 0 \)) \( \to \) S1$^{\prime}$ (\( \rho_{t}= 0.7 \))}%
	                
	                \medskip
	                
	                \makebox[.5\linewidth]{\includegraphics[width=.45\linewidth, height=0.36\textwidth]{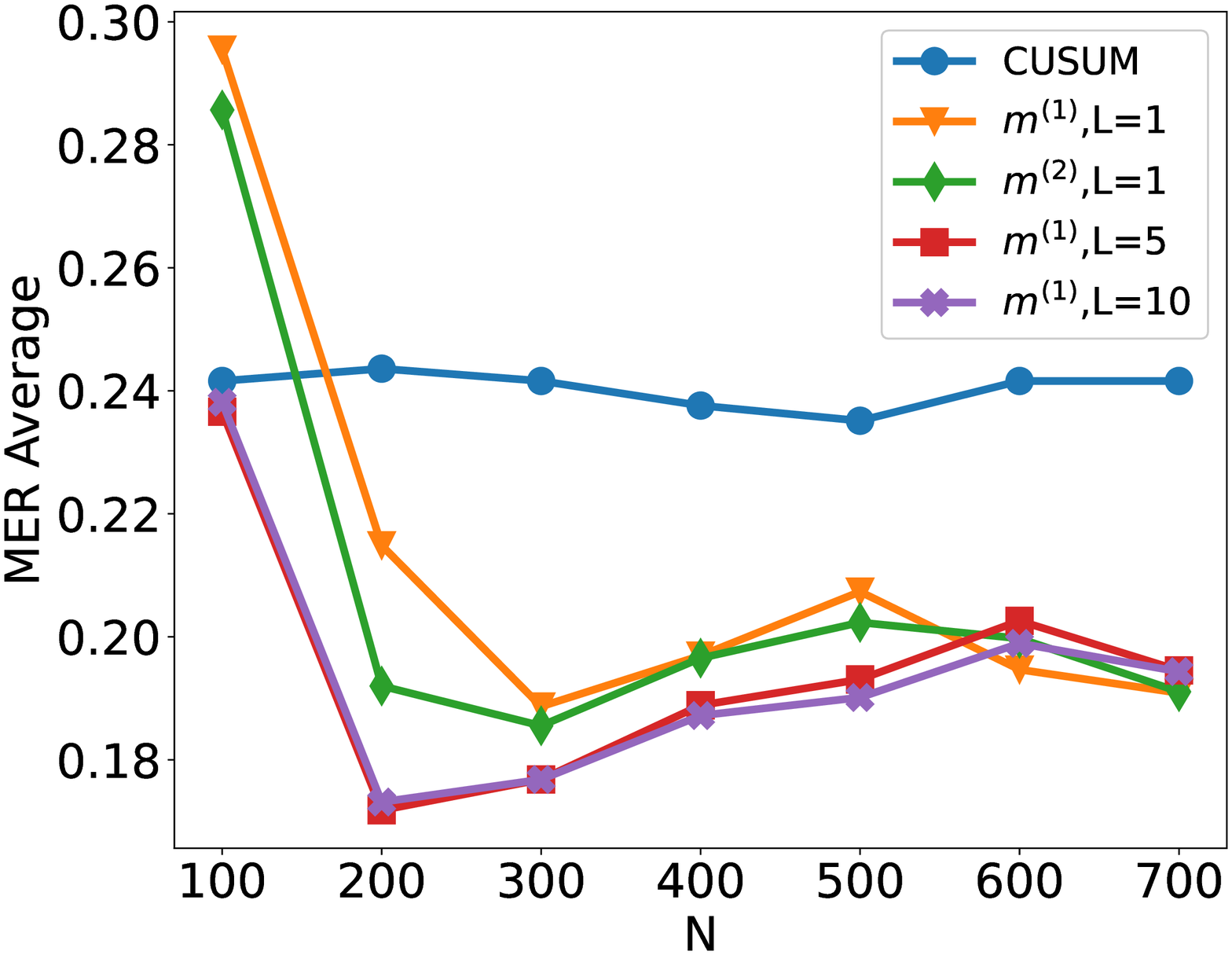}}
	                \makebox[.5\linewidth]{\includegraphics[width=.45\linewidth, height=0.36\textwidth]{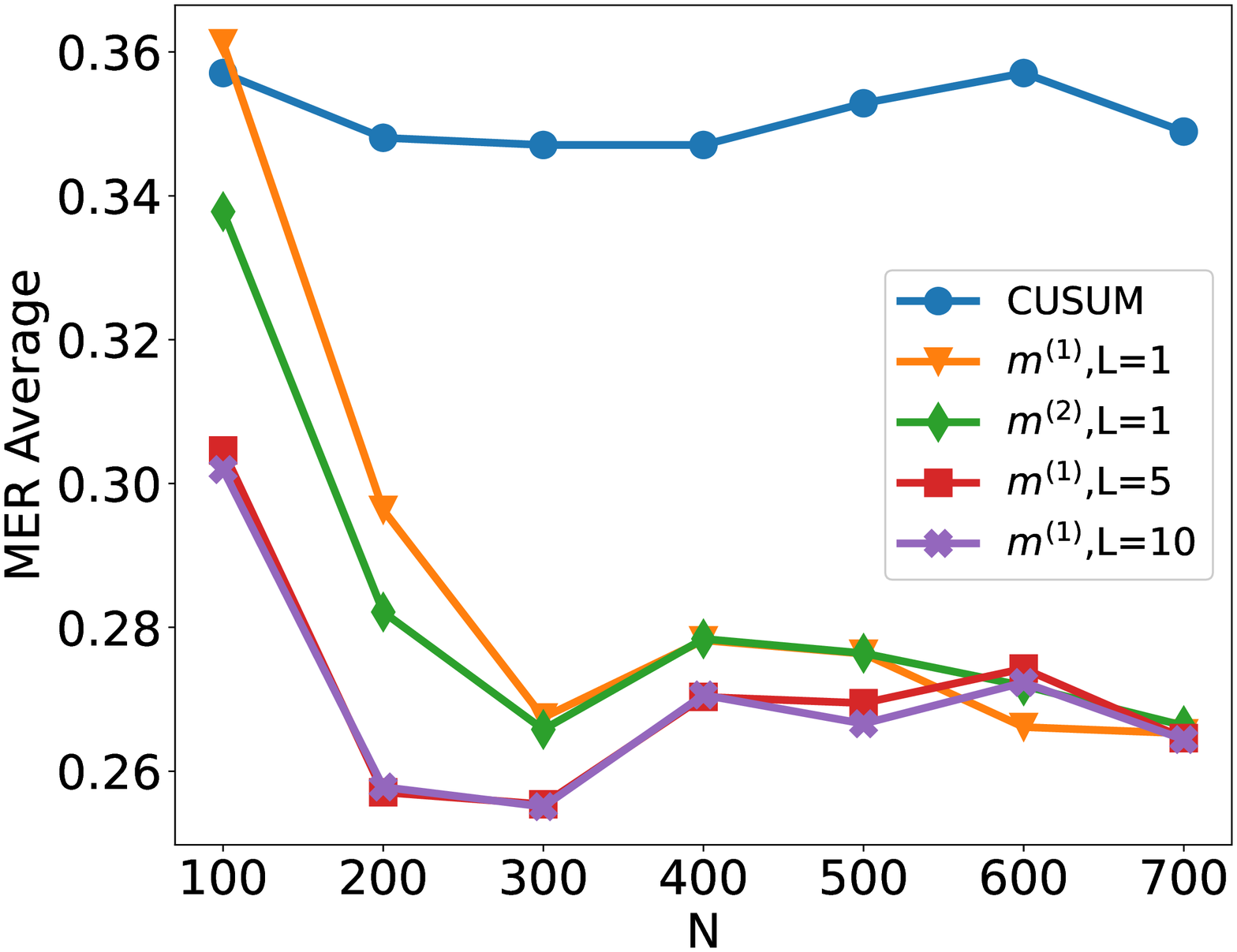}}
	                \makebox[.5\linewidth]{\small (c) Trained S1 (\( \rho_{t}= 0 \)) \( \to \) S2}%
	                \makebox[.5\linewidth]{\small (d) Trained S1 (\( \rho_{t}= 0 \)) \( \to \) S3}%
	            \end{minipage}
	            \caption{Plot of the test set MER, computed on a test set of size $N_{\mathrm{test}}=30000$,  against training sample size $N$ for detecting the existence of a change-point on data series of length $n=100$. We compare the performance of the CUSUM test and neural networks from four function classes: $\mathcal{H}_{1,m^{(1)}}$,$\mathcal{H}_{1,m^{(2)}}$, $\mathcal{H}_{5,m^{(1)}\mathbf{1}_{5}}$ and $\mathcal{H}_{10,m^{(1)}\mathbf{1}_{10}}$ where $m^{(1)} = 4\lfloor\log_2(n)\rfloor$ and $m^{(2)} = 2n-2$ respectively under scenarios S1, S1$^\prime$, S2 and S3 described in~\Cref{sec:Simulation_Study}. The subcaption ``A \( \to \) B'' means that we apply the trained classifier ``A'' to target testing dataset ``B''.}\label{fig:simulation_resultA2Others}
	        \end{figure}
	        
	    \subsubsection{Simulation for change in autocorrelation}\label{subsubsec:Simulation_for_change_in_autocorrelation}
	        
	        In this simulation, we discuss how we can use neural networks to recreate test statistics for various types of changes. For instance, if the data follows an AR(1) structure, then changes in autocorrelation can be handled by including transformations of the original input of the form $(x_{t}x_{t+1})_{t=1,\ldots,n-1}$. On the other hand, even if such transformations are not supplied as the input, a deep neural network of suitable depth is able to approximate these transformations and consequently successfully detect the change~\citep[Lemma~A.2]{schmidt-hieberNonparametricRegressionUsing2020}.  This is illustrated in~\Cref{fig:AR}, where we compare the performance of neural network based classifiers of various depths constructed with and without using the transformed data as inputs. 
	        \begin{figure}[htbp]
	            \begin{minipage}{1\linewidth}
	                \makebox[.5\linewidth]{\includegraphics[width=.45\linewidth]{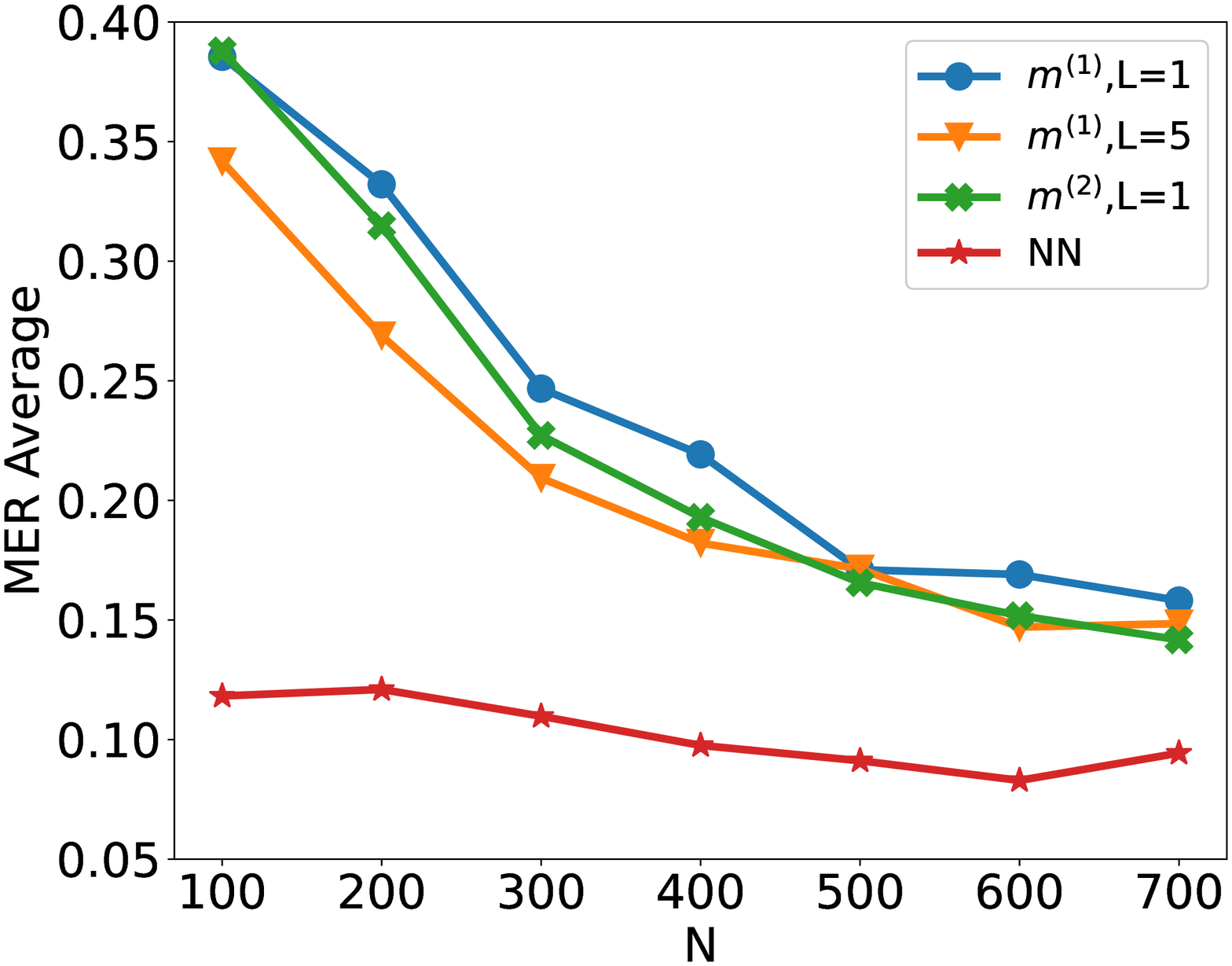}}%
	                \makebox[.5\linewidth]{\includegraphics[width=.45\linewidth]{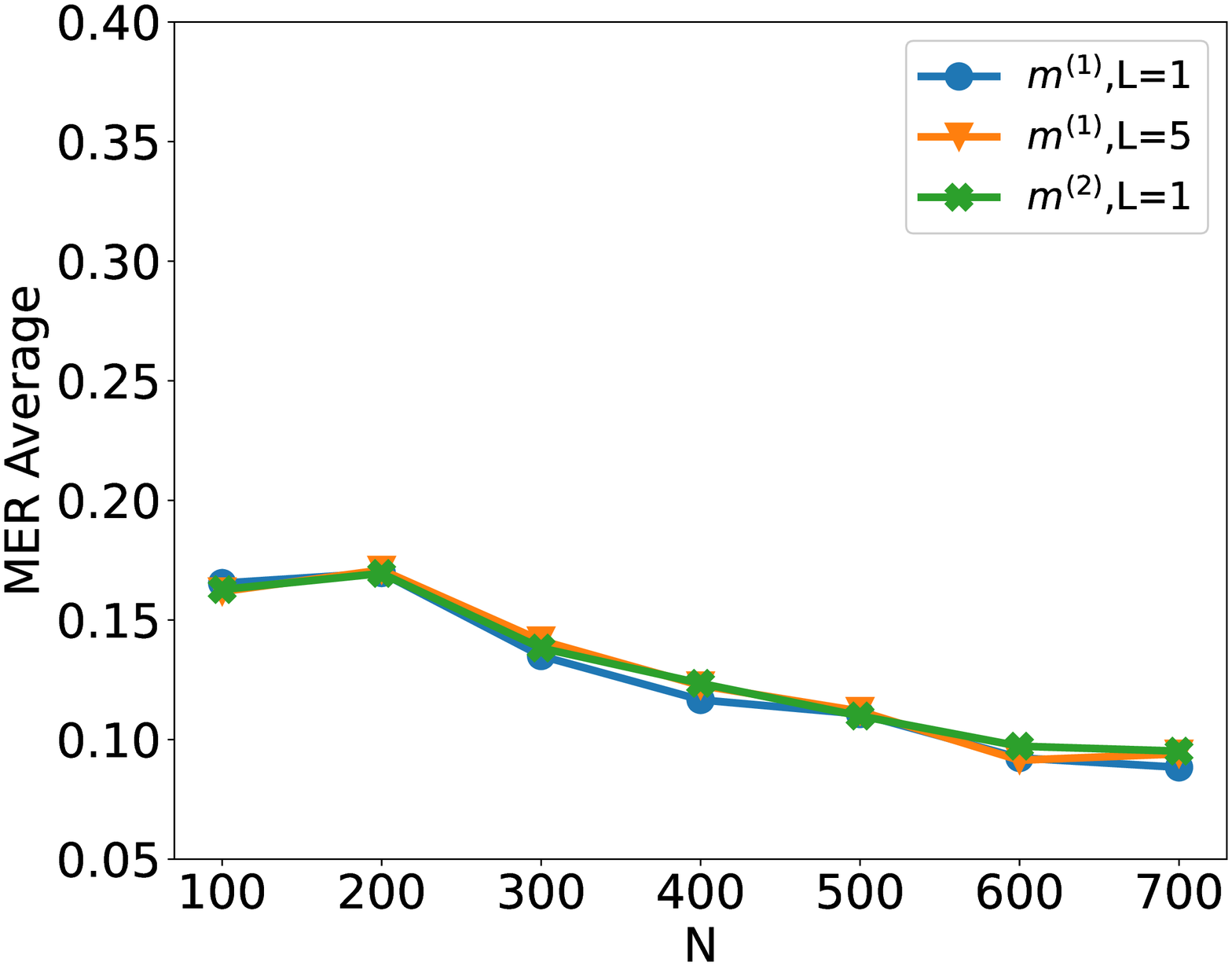}}
	                \makebox[.5\linewidth]{\small (a) Original Input}%
	                \makebox[.5\linewidth]{\small (b) Original and $x_{t}x_{t+1}$ Input}%
	            \end{minipage}
	            \caption{Plot of the test set MER, computed on a test set of size $N_{\mathrm{test}}=30000$,  against training sample size $N$ for detecting the existence of a change-point on data series of length $n=100$. We compare the performance of  neural networks from four function classes: $\mathcal{H}_{1,m^{(1)}}$,$\mathcal{H}_{1,m^{(2)}}$, $\mathcal{H}_{5,m^{(1)}\mathbf{1}_{5}}$ and neural network with 21 residual blocks where $m^{(1)} = 4\lfloor\log_2(n)\rfloor$ and $m^{(2)} = 2n-2$ respectively.  The change-points are randomly chosen from \( \mathrm{Unif}\{10,\ldots,89\} \). Given change-point $\tau$, data are generated from the autoregressive model $x_t = \alpha_t x_{t-1} + \epsilon_t$ for $\epsilon_t\stackrel{\mathrm{iid}}{\sim} N(0, 0.25^2)$ and $\alpha_t = 0.2\mathbf{1}_{\{t<\tau\}} + 0.8\mathbf{1}_{\{t\geq \tau\}}$.}\label{fig:AR}
	        \end{figure}
	    
	    \subsubsection{Simulation on change-point location estimation}
	        Here, we describe simulation results on the performance of change-point location estimator constructed using a combination of simple neural network-based classifier and~\Cref{alg:change_point_detection} from the main text.   Given a sequence of length \(n^{\prime}= 2000 \), we draw \( \tau\sim\text{Unif}\{750,\ldots,1250\} \). Set \( \mu_{L}=0 \) and draw \( \mu_{R}|\tau\) from 2 different uniform distributions: \( \text{Unif}([-1.5b,-0.5b]\cup [0.5b, 1.5b])  \) (Weak) and \( \text{Unif}([-3b,-b]\cup [b, 3b])  \) (Strong), where \( b\coloneqq\sqrt{\frac{8n^{\prime}\log(20n^{\prime})}{\tau(n^{\prime}-\tau)}}\) is chosen in line with~\Cref{lem:hypothesis_test} to ensure a good range of signal-to-noise ratio. We then generate $\boldsymbol{x} = (\mu_{\mathrm{L}}\mathbbm{1}_{\{t\leq \tau\}} + \mu_{\mathrm{R}}\mathbbm{1}_{\{t> \tau\}} + \varepsilon_{t})_{t\in[n^{\prime}]}$, with the noise $\boldsymbol{\varepsilon} = (\varepsilon_{t})_{t\in[n^{\prime}]}\sim N_{n'}(0,I_{n'})$. We then draw independent copies $\boldsymbol{x}_1,\ldots,\boldsymbol{x}_{N'}$ of $\boldsymbol{x}$. For each \( \boldsymbol{x}_{k} \), we randomly choose 60 segments with length \( n\in\{300, 400, 500, 600\} \), the segments which include \( \tau_{k} \) are labelled `1', others are labelled `0'. The training dataset size is \( N=60N^{\prime} \) where \( N^{\prime}=500\). We then draw another $N_{\text{test}}=3000$ independent copies of $\boldsymbol{x}$ as our test data for change-point location estimation. We study the performance of change-point location estimator produced by using~\Cref{alg:change_point_detection} together with a single-layer neural network, and compare it with the performance of CUSUM, MOSUM and Wilcoxon statistics-based estimators.  As we can see from the~\Cref{fig:rmse}, under Gaussian models where CUSUM is known to work well, our simple neural network-based procedure is competitive. On the other hand, when the noise is heavy-tailed, our simple neural network-based estimator greatly outperforms CUSUM-based estimator.
	        \begin{figure}[htbp]
	            \begin{minipage}{1\linewidth}
	                \makebox[.5\linewidth]{\includegraphics[width=.45\linewidth]{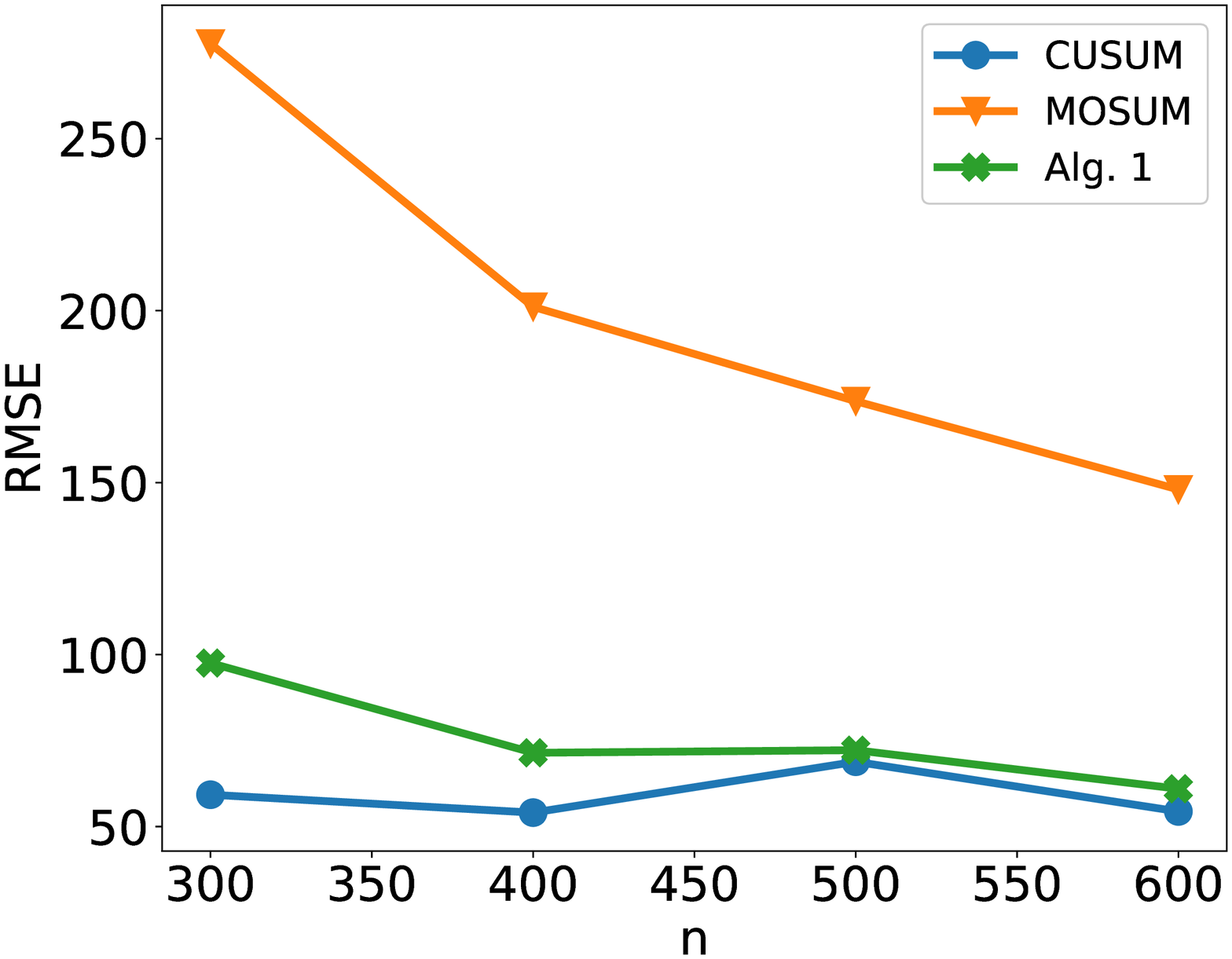}}%
	                \makebox[.5\linewidth]{\includegraphics[width=.45\linewidth]{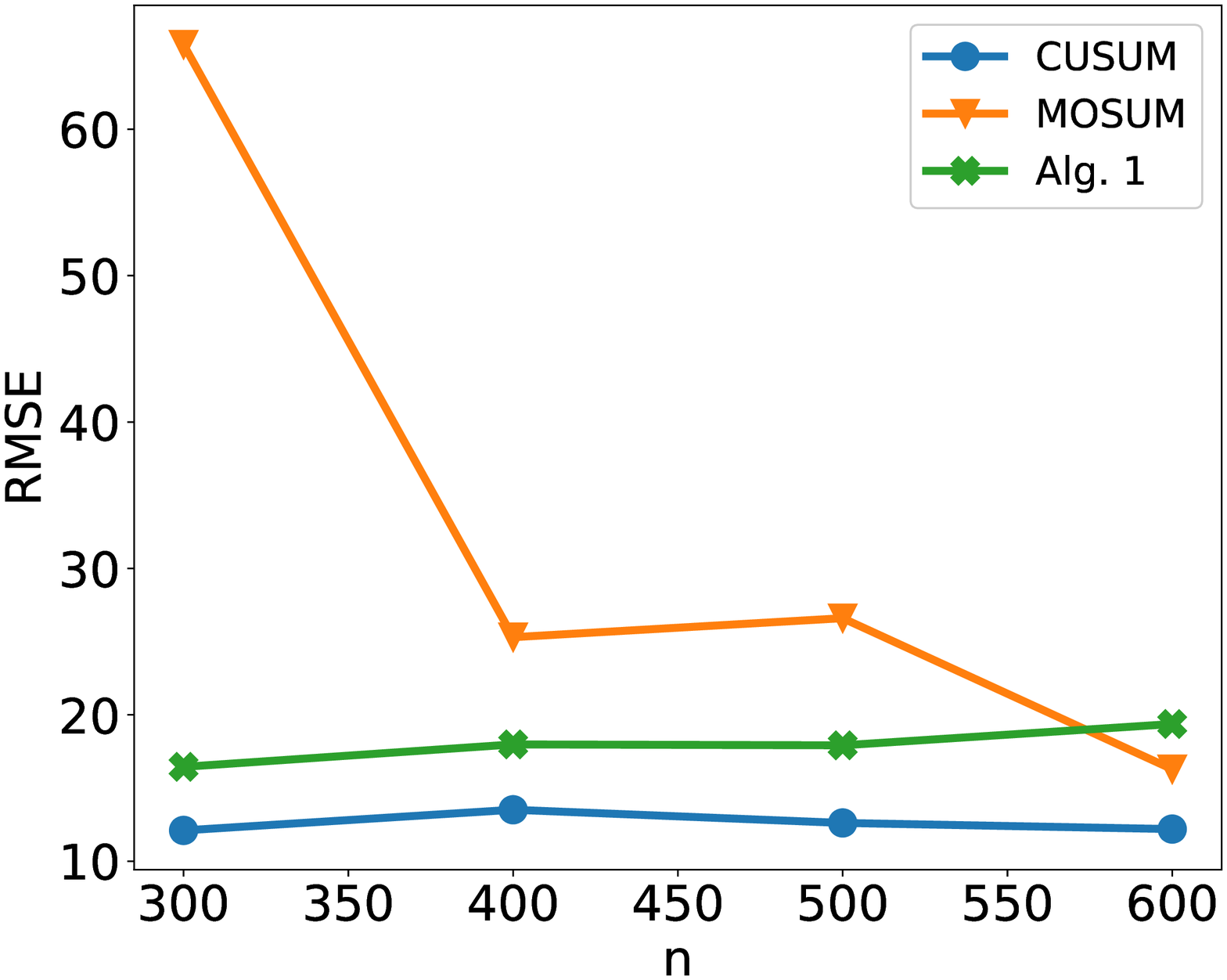}}
	                \makebox[.5\linewidth]{\small (a) S1  with \( \rho_{t}=0 \), weak SNR}%
	                \makebox[.5\linewidth]{\small (b) S1  with \( \rho_{t}=0 \), strong SNR}%
	                
	                \medskip
	                \makebox[.5\linewidth]{\includegraphics[width=.45\linewidth]{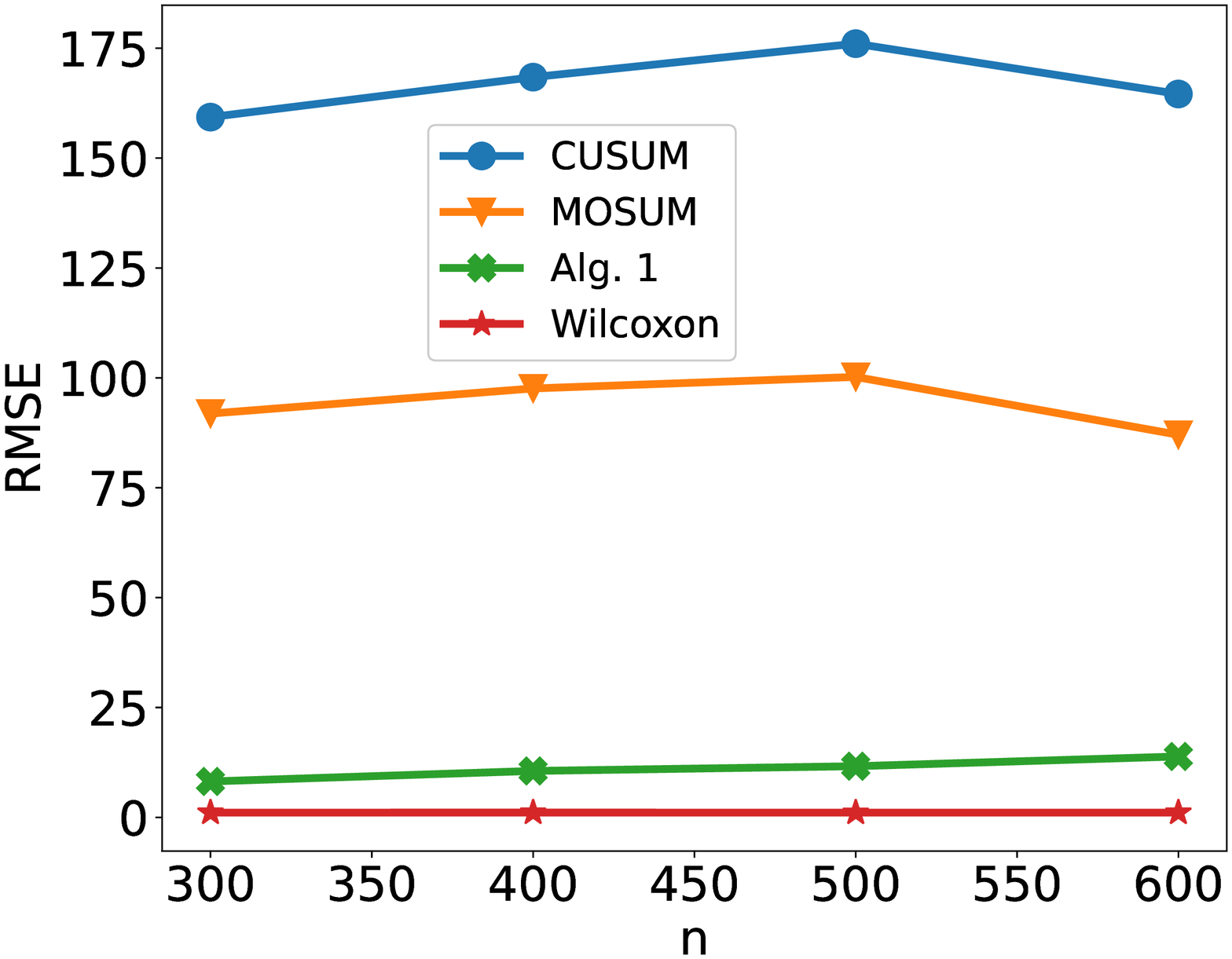}}%
	                \makebox[.5\linewidth]{\includegraphics[width=.45\linewidth]{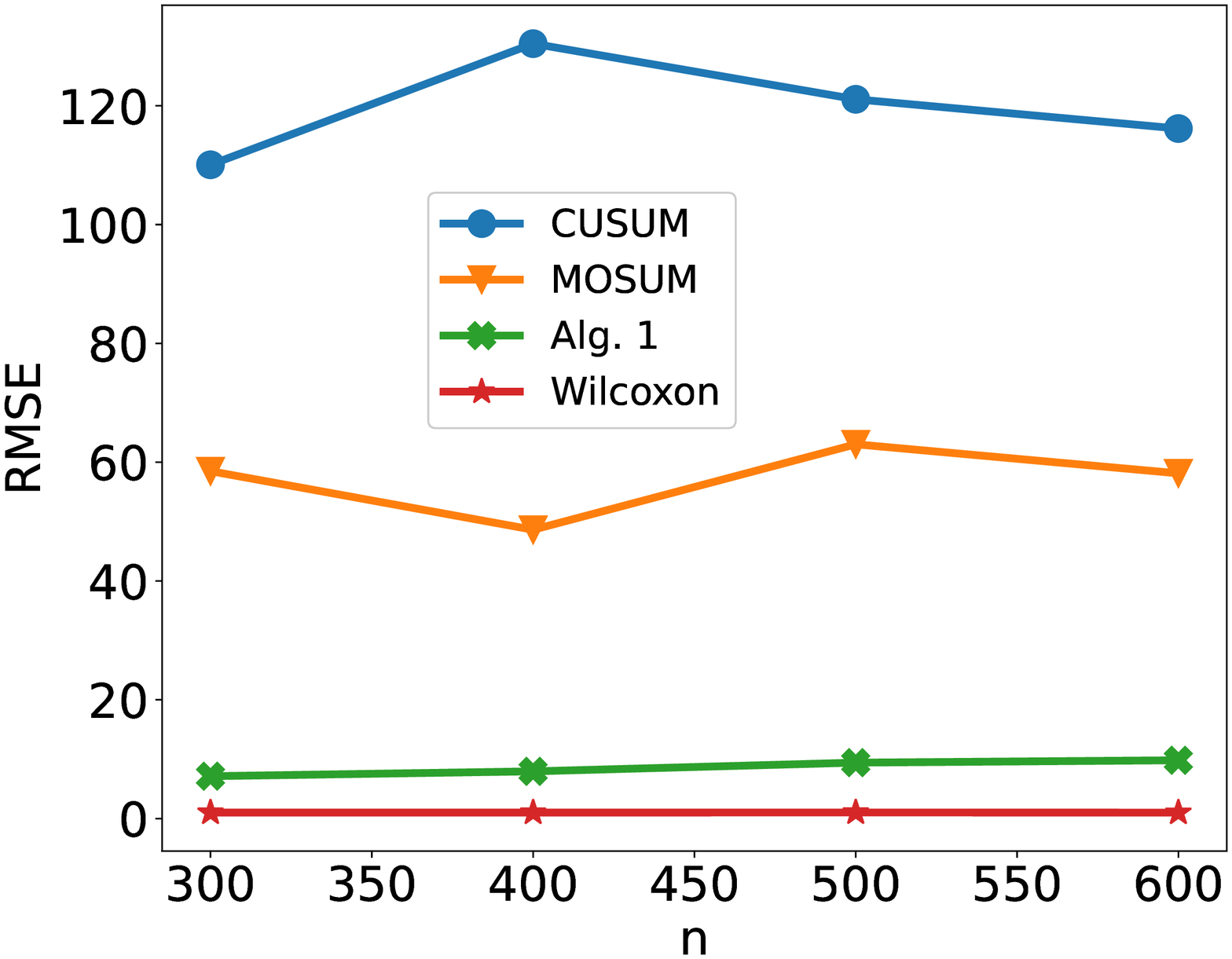}}
	                \makebox[.5\linewidth]{\small (c) S3, weak SNR}%
	                \makebox[.5\linewidth]{\small (d) S3, strong SNR}%
	            \end{minipage}
	            \caption{Plot of the root mean square error (RMSE)  of change-point estimation (S1 with \( \rho_{t}=0 \) and S3), computed on a test set of size $N_{\text{test}}=3000$,  against bandwidth $n$ for detecting the existence of a change-point on data series of length $n^{*}=2000$. We compare the performance of the change-point detection by CUSUM, MOSUM,~\Cref{alg:change_point_detection} and Wilcoxon (only for S3) respectively. The RMSE here is defined by \(\sqrt{1/N \sum_{i=1}^{N}(\hat{\tau}_{i}-\tau_{i})^{2}} \) where \( \hat{\tau}_{i} \) is the estimator of change-point for the \( i \)-th observation and \( \tau_{i} \) is the true change-point. The weak and strong signal-to-noise ratio (SNR) correspond to \( \mu_{R}|\tau\sim\text{Unif}([-1.5b,-0.5b]\cup [0.5b, 1.5b])  \) and \( \mu_{R}|\tau\sim\text{Unif}([-3b,-b]\cup [b, 3b])  \) respectively.}\label{fig:rmse}
	        \end{figure}
	        \clearpage
\section{Real Data Analysis}\label{sec:More_Details_of_Numerical_Study_and_Real_Data_Analysis}
	The HASC (Human Activity Sensing Consortium) project aims at understanding the human activities based on the sensor data. This data includes 6 human activities: ``stay'', ``walk'', ``jog'', ``skip'', ``stair up'' and ``stair down''. Each activity lasts at least 10 seconds,  the sampling frequency is 100 Hz.
	\subsection{Data Cleaning}
	    The HASC offers sequential data where there are multiple change-types and multiple change-points, see~\Cref{fig:figures/HASC2011.eps} in main text. Hence, we can not directly feed them into our deep convolutional residual neural network. The training data fed into our neural network requires fixed length $n$ and either one change-point or no change-point existence in each time series. Next, we describe how to obtain this kind of training data from HASC sequential data. In general, Let \( \boldsymbol{x}={(x_{1},x_{2},\ldots,x_{d})}^{\top}, d\geq1  \) be the \( d \)-channel vector.
	    Define \( \boldsymbol{X}\coloneqq (\boldsymbol{x}_{t_{1}},\boldsymbol{x}_{t_{2}},\ldots,\boldsymbol{x}_{t_{n^{*}}})\) as  a realization of \( d \)-variate time series where \( \boldsymbol{x}_{t_{j}}, j=1,2,\ldots,n^{*} \) are the observations of  \( \boldsymbol{x} \) at \( n^{*} \) consecutive time stamps \( t_{1},t_{2},\ldots,t_{n^{*}} \). Let \( \boldsymbol{X}_{i}, i=1,2,\ldots,N^* \) represent the observation from the \(i\)-th subject.\ \( \boldsymbol{\tau}_{i}\coloneqq(\tau_{i,1},\tau_{i,2},\ldots,\tau_{i,K})^{\top}, K\in \mathbb{Z}^{+}, \tau_{i,k}\in[2,n^{*}-1], 1\leq k\leq K \) with convention \( \tau_{i,0}=0 \) and \( \tau_{i,K+1}=n^{*} \) represents the change-points of the \( i\)-th observation which are well-labelled in the sequential data sets. Furthermore, define \( n\coloneqq\min_{i\in[N^{*}]}\min_{k\in[K+1]}(\tau_{i,k}-\tau_{i,k-1}) \). In practice, we require that \( n \) is not too small, this can be achieved by controlling the sampling frequency in experiment, see HASC data. We randomly choose \( q \) sub-segments with length \( n \) from \( \boldsymbol{X}_{i} \) like  the gray dash rectangles in~\Cref{fig:figures/HASC2011.eps} of main text. By the definition of \( n \), there is at most one change-point in each sub-segment. Meanwhile, we assign the label to each sub-segment according to the type and existence of change-point. After that, we stack all the sub-segments to form a tensor \( \mathcal{X} \) with dimensions of \( (N^{*}q, d, n) \). The label vector is denoted as \( \mathcal{Y} \) with length \( N^{*}q \).
	    To guarantee that there is at most one change-point in each segment,  we set the length of segment \( n=700 \). Let \( q=15 \), as the change-points are well labelled, it is easy to draw 15 segments without any change-point, i.e., the segments with labels: ``stay'', ``walk'', ``jog'', ``skip'', ``stair up'' and ``stair down''. Next, we randomly draw 15 segments (the red rectangles in~\Cref{fig:figures/HASC2011.eps} of main text) for each transition point.
	\subsection{Transformation}
	    
	    \Cref{sec:CUSUM test and its generalisations are neural networks} in main text suggests that changes in the mean/signal may be captured by feeding the raw data directly. For other type of change, we recommend appropriate transformations before training the model depending on the interest of change-type. For instance, if we are interested in changes in the second order structure, we suggest using the square transformation; for change in auto-correlation with order \( p \) we could input the cross-products of data up to a \( p \)-lag. In multiple change-types, we allow applying several transformations to the data in data pre-processing step. The mixture of raw data and transformed data is treated as the training data.

	    We employ the square transformation here. All the segments are mapped onto scale \( [-1,1] \) after the transformation.
	    The frequency of training labels are list in~\Cref{fig:figures/label_freq.eps}. Finally, the shapes of training and test data sets are \( (4875, 6, 700) \) and \( (1035, 6, 700) \) respectively.
	
	\subsection{Network Architecture}\label{subsec:Network_Architechure}
	    We propose a general deep convolutional residual neural network architecture to identify the multiple change-types based on the residual block technique~\citep{heDeepResidualLearning2016} (see~\Cref{fig:Arch-ResNet.eps}). There are two reasons to explain why we choose residual block as the  skeleton frame.
	    \begin{itemize}
	        \item The problem of vanishing gradients~\citep{bengioLearningLongtermDependencies1994,glorotUnderstandingDifficultyTraining2010}. As the number of convolution layers goes significantly deep, some layer weights might vanish in back-propagation which hinders the convergence. Residual block can solve this issue by the so-called ``shortcut connection'', see the flow chart in~\Cref{fig:Arch-ResNet.eps}.
	        \item Degradation.~\citet{heDeepResidualLearning2016} has pointed out that when the number of convolution layers increases significantly, the accuracy might get saturated and degrade quickly. This phenomenon is reported and verified in~\citet{he2015convolutional} and~\citet{heDeepResidualLearning2016}.
	    \end{itemize}
	    
	    \begin{figure}[htbp]
	        \centering
	        \makebox{\includegraphics[width=0.78\textwidth,height=0.28\textwidth]{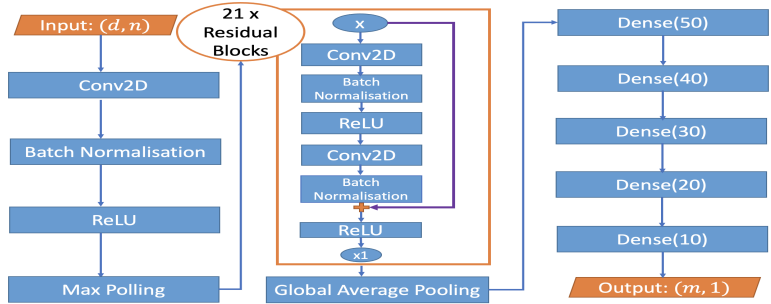}}
	        \caption{Architecture of our general-purpose change-point detection neural network. The left column shows the standard layers of neural network with input size $(d,n)$, \( d \) may represent the number of transformations or channels; We use 21 residual blocks and one global average pooling in the middle column; The right column includes 5 dense layers with nodes in bracket and output layer.  More details of the neural network architecture appear in the supplement.}\label{fig:Arch-ResNet.eps}
	    \end{figure}
	    
	    There are 21 residual blocks in our deep neural network, each residual block contains 2 convolutional layers.  Like the suggestion in~\citet{ioffeBatchNormalizationAccelerating2015} and~\citet{heDeepResidualLearning2016}, each convolution layer is followed by one Batch Normalization (BN) layer and one ReLU layer. Besides, there exist 5 fully-connected convolution layers right after the residual blocks, see the third column of~\Cref{fig:Arch-ResNet.eps}. For example, \textbf{Dense(50)} means that the dense layer has 50 nodes and is connected to a dropout layer with dropout rate 0.3. To further prevent the effect of overfitting, we also implement the \( L_{2} \) regularization in each fully-connected layer~\citep{ngFeatureSelectionL12004}. As the number of labels in HASC is 28, see~\Cref{fig:figures/label_dict.eps}, we drop the dense layers ``Dense(20)'' and ``Dense(10)'' in~\Cref{fig:Arch-ResNet.eps}. The output layer has size \( (28,1) \).
	    
	    We remark two discussable issues here.\ (a) For other problems, the number of residual blocks, dense layers and the hyperparameters may vary depending on the complexity of the problem. In~\Cref{sec:Detecting multiple changes and multiple change-types} of main text, the architecture of neural network for both synthetic data and real data has 21 residual blocks considering the trade-off between time complexity and model complexity.  Like the suggestion in~\citet{heDeepResidualLearning2016}, one can also add more residual blocks into the architecture to improve the accuracy of classification.\ (b) In practice, we would not have enough training data; but there would be potential ways to overcome  this via either using Data Argumentation or increasing \( q \). In some extreme cases that we only mainly have data with no-change, we can artificially add changes into such data in line with the type of change we want to detect.
	\subsection{Training and Detection}\label{Training and Detection}
	    \begin{figure}[ht]
	        \centering
	        \makebox{\includegraphics[width=0.9\textwidth]{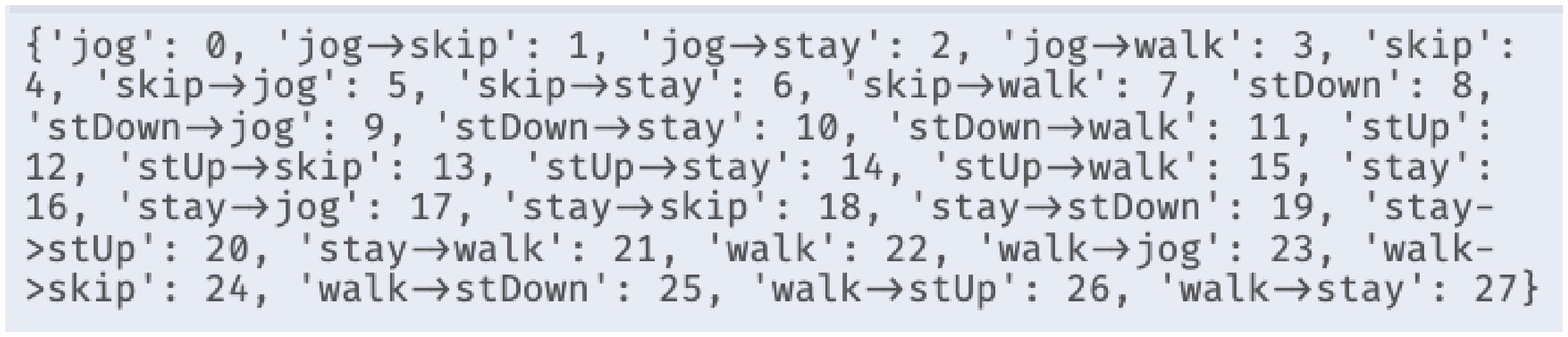}}
	        \caption{Label Dictionary}\label{fig:figures/label_dict.eps}
	    \end{figure}
	    \begin{figure}[ht]
	        \centering
	        \makebox{\includegraphics[width=0.9\textwidth]{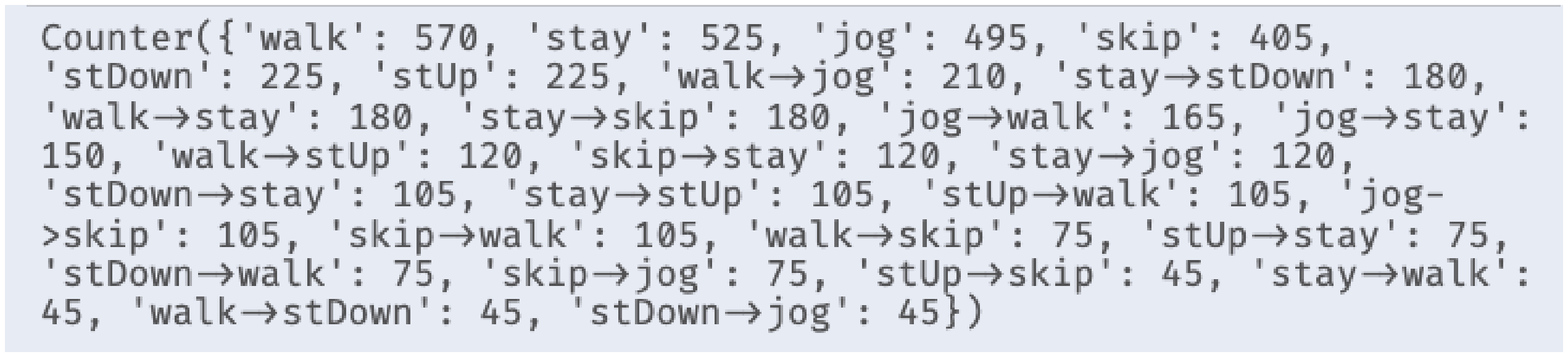}}
	        \caption{Label Frequency}\label{fig:figures/label_freq.eps}
	    \end{figure}
	    \begin{figure}[ht]
	        \centering
	        \makebox{\includegraphics[width=0.6\textwidth]{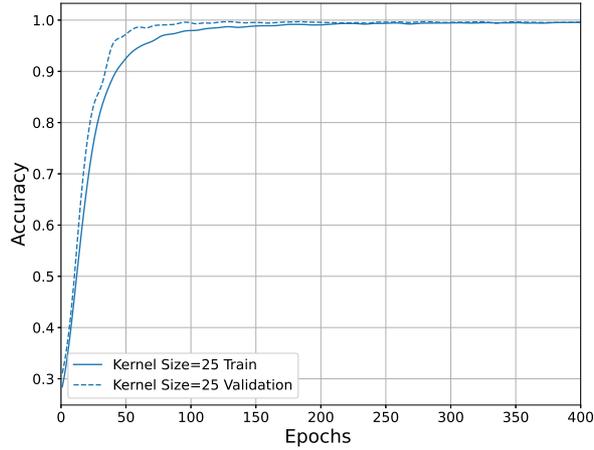}}
	        \caption{The Accuracy Curves}\label{fig:figures/RealDataKS25+acc.eps}
	    \end{figure}
	    \begin{figure}[ht]
	        \centering
	        \makebox{\includegraphics[width=0.8\textwidth]{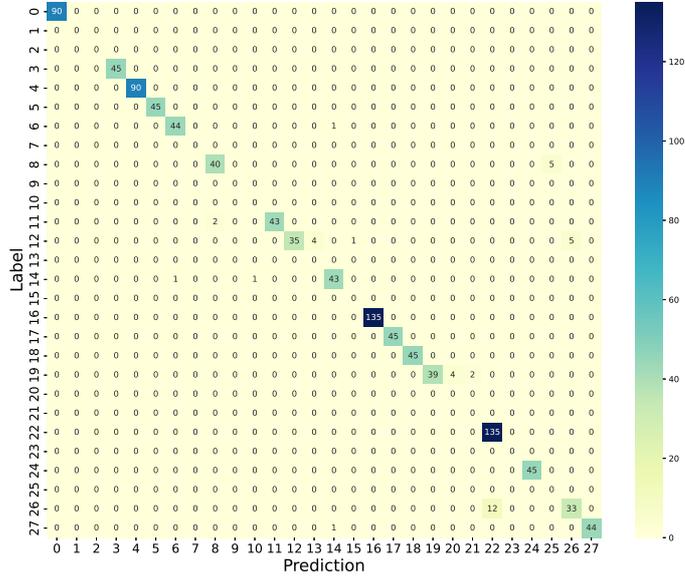}}
	        \caption{Confusion Matrix of Real Test Dataset}\label{fig:figures/RealDataKS25Confusion_matrix.eps}
	    \end{figure}
	    \begin{figure}[ht]
	        \centering
	        \makebox{\includegraphics[width=0.9\textwidth]{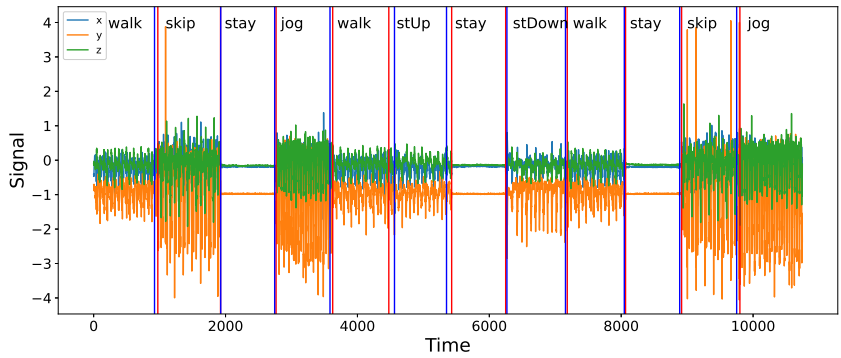}}
	        \caption{Change-point Detection of Real Dataset for Person 7 (2nd sequence). The red line at 4476 is the true change-point, the blue line on its right is the estimator. The difference between them is caused by the similarity of ``Walk'' and ``StairUp''.}\label{fig:figures/RealDataCPDEst2.eps}
	    \end{figure}
	    \begin{figure}[ht]
	        \centering
	        \makebox{\includegraphics[width=0.9\textwidth]{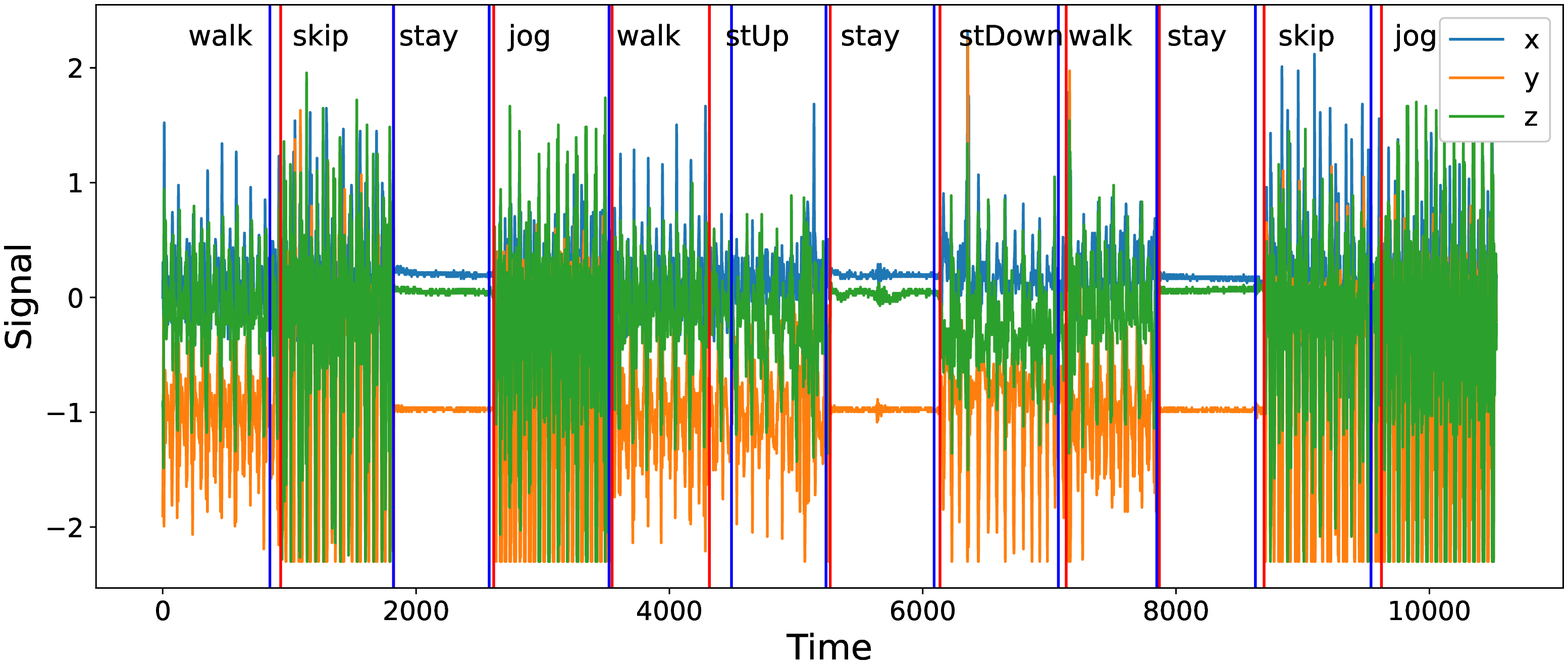}}
	        \caption{Change-point Detection of Real Dataset for Person 7 (3rd sequence). The red vertical lines represent the underlying change-points, the blue vertical lines represent the estimated change-points.}\label{fig:figures/RealDataCPDEst3.eps}
	    \end{figure}
	    
	    There are 7 persons observations in this dataset. The first 6 persons sequential data are treated as the training dataset, we use the last person's data to validate the trained classifier. Each person performs each of 6 activities: ``stay'', ``walk'', ``jog'', ``skip'', ``stair up'' and ``stair down''  at least 10 seconds.
	    The transition point between two consecutive activities can be treated as the change-point. Therefore, there are 30 possible types of change-point. The total number of labels is 36 (6 activities and 30 possible transitions). However, we only found 28 different types of label in this real dataset, see~\Cref{fig:figures/label_dict.eps}.

	    The initial learning rate is 0.001, the epoch size is 400. Batch size is 16, the dropout rate is 0.3, the filter size is 16 and the kernel size is \( (3,25) \). Furthermore, we also use 20\% of the training dataset to validate the classifier during training step.
	    
	    \Cref{fig:figures/RealDataKS25+acc.eps} shows the accuracy curves of training and validation. After 150 epochs, both solid and dash curves approximate to 1. The test accuracy is 0.9623, see the confusion matrix in~\Cref{fig:figures/RealDataKS25Confusion_matrix.eps}. These results show that our neural network classifier performs well both in the training and test datasets.
	    
	    Next, we apply the trained classifier to 3 repeated sequential datasets of Person 7 to detect the change-points. The first sequential dataset has shape \( (3, 10743) \). First,  we extract the \( n \)-length sliding windows with stride 1 as the input dataset. The input size becomes \( (9883, 6, 700) \). Second,  we use~\Cref{alg:change_point_detection} to detect the change-points where we relabel the activity label as ``no-change'' label and transition label as ``one-change'' label respectively.
	    ~\Cref{fig:figures/RealDataCPDEst2.eps,fig:figures/RealDataCPDEst3.eps} show the results of multiple change-point detection for other 2 sequential data sets from the 7-th person.
\end{document}